\documentclass[final, 12pt]{colt2018} 

\title[More Adaptive Algorithms for Adversarial Bandits]{More Adaptive Algorithms for Adversarial Bandits}
\usepackage{float}

\usepackage{algorithm}
\usepackage{times}
\usepackage{makecell}
\usepackage[final]{showlabels}
\usepackage{threeparttable}
\allowdisplaybreaks
\usepackage{bbm}
\usepackage{enumerate}
\newcommand{\p}{\prime}
\DeclareMathOperator*{\argmin}{argmin}

\newcommand{\inner}[1]{ \left\langle {#1} \right\rangle }
\newcommand{\inn}[1]{ \langle {#1} \rangle }
\newcommand{\absolute}[1]{ \left\lvert {#1} \right\rvert }
\newcommand{\abs}[1]{ \lvert {#1} \rvert }
\usepackage{amsmath}
\newcommand\norm[1]{\left\lVert#1\right\rVert}
\newcommand{\reg}{\text{\rm Reg}}
\newtheorem{cor}[theorem]{Corollary}

 \coltauthor{\Name{Chen-Yu Wei} \Email{chenyu.wei@usc.edu} \\
 \addr University of Southern California
 \AND
 \Name{Haipeng Luo} \Email{haipengl@usc.edu}\\
 \addr University of Southern California
 }

\begin{document}

\maketitle

\begin{abstract}
We develop a novel and generic algorithm for the adversarial multi-armed bandit problem (or more generally the combinatorial semi-bandit problem).
When instantiated differently, our algorithm achieves various new data-dependent regret bounds improving previous work.
Examples include:
1) a regret bound depending on the variance of only the best arm;
2) a regret bound depending on the first-order path-length of only the best arm;
3) a regret bound depending on the sum of the first-order path-lengths of all arms as well as an important negative term, which together lead to faster convergence rates for some normal form games with partial feedback;
4) a regret bound that simultaneously implies small regret when the best arm has small loss {\it and} logarithmic regret when there exists an arm whose expected loss is always smaller than those of other arms by a fixed gap (e.g. the classic i.i.d. setting).
In some cases, such as the last two results, our algorithm is completely parameter-free.

The main idea of our algorithm is to apply the optimism and adaptivity techniques to the well-known Online Mirror Descent framework
with a special log-barrier regularizer. 
The challenges are to come up with appropriate optimistic predictions and correction terms in this framework.
Some of our results also crucially rely on using a sophisticated increasing learning rate schedule.
\end{abstract}


\section{Introduction}

The adversarial Multi-Armed Bandits (MAB) problem~\citep{auer2002nonstochastic} is a classic online learning problem with partial information feedback.
In this problem, at each round the learner selects one of the $K$ arms while simultaneously the adversary decides the loss of each arm,
then the learner suffers and observes (only) the loss of the picked arm.
The goal of the learner is to minimize the regret, that is, the difference between her total loss and the total loss of the best fixed arm.
The classic Exp3 algorithm~\citep{auer2002nonstochastic} achieves a regret bound of order $\tilde{\mathcal{O}}(\sqrt{TK})$ after $T$ rounds,\footnote{%
Throughout the paper we use the notation $\tilde{\mathcal{O}}(\cdot)$ to suppress factors that are poly-logarithmic in $T$ and $K$.}
which is worst-case optimal up to logarithmic factors.

There are several existing works on deriving more adaptive bandit algorithms, replacing the dependence on $T$ in the regret bound
by some data-dependent quantity that is $\mathcal{O}(T)$ in the worst-case but could be potentially much smaller in benign environments.
Examples of such data-dependent quantities include the loss of the best arm~\citep{allenberg2006hannan, foster2016learning} or the empirical variance of all arms~\citep{hazan2011better, bubeck2017sparsity}. 
Extensions to more general settings such as semi-bandit, two-point bandit, and graph bandit have also been studied~\citep{neu2015first, chiang2013beating, lykouris2017small}.
These adaptive algorithms not only enjoy better performance guarantees,
but also have important applications for other areas such as game theory~\citep{foster2016learning}.

In this work, we propose a novel and generic bandit algorithm in the more general semi-bandit setting
(formally defined in Section~\ref{section:notations}).
By instantiating this generic algorithm differently,
we obtain various adaptive algorithms with new data-dependent expected regret bounds that improve previous work.
When specified to the MAB setting with $\ell_{t,i} \in [-1,1]$ denoting the loss of arm $i$ at time $t$ (and $\ell_{0,i} \triangleq 0$),
these bounds replace the dependence on $T$ by (also see Table~\ref{table:summary} for a summary):
\begin{itemize}
\item 
$\sum_{t=1}^T (\ell_{t,i^\star} - \frac{1}{T}\sum_{s=1}^T\ell_{s,i^\star})^2$,
that is, the (unnormalized) variance of the best arm $i^\star$.
Similar existing bounds of~\citep{hazan2011better, hazan2011simple, bubeck2017sparsity} replace $T$ by the average of the variances of all arms.
In general these two are incomparable. 
However, note that the variance of the best arm is always bounded by $K$ times the average variance,
while it is possible that the latter is of order $\Theta(T)$ and the former is only $\mathcal{O}(1)$. (Section~\ref{subsubsection:variation bound})

\item
$K\sum_{t=1}^T |\ell_{t,i^\star} - \ell_{t-1,i^\star}|$, that is, ($K$ times) the first-order path-length of the best arm. (Section~\ref{subsubsection:path-length})

\item
$\sum_{i=1}^K\sum_{t=1}^T |\ell_{t,i} - \ell_{t-1,i}|$, that is, the sum of the first-order path-lengths of all arms.
Importantly, there is also an additional negative term in the regret similar to the one of~\citep{syrgkanis2015fast} for the full information setting.
This implies a fast convergence rate of order $1/T^\frac{3}{4}$ for several game playing settings with bandit feedback. 
(Sections~\ref{subsection:first_order_better_k})

\item
A new quantity in terms of some second-order excess loss (see Eq.~\eqref{eqn:new_excess_loss_bound} for the exact form). 
While the bound is not easy to interpret on it own, 
it in fact automatically and simultaneously implies the so-called ``small-loss'' bound $\tilde{\mathcal{O}}\Big(\sqrt{K\sum_{t=1}^T \ell_{t,i^\star}}\Big)$,\footnote{%
Assuming that losses are non-negative in this case as it is common for small-loss bounds.
}
{\it and} logarithmic regret $\mathcal{O}(\frac{K\ln T}{\Delta})$ if there is an arm whose expected loss is always smaller than those of other arms by a fixed gap $\Delta$
(e.g. the classic i.i.d. MAB setting~\citep{lai1985asymptotically}). (Section~\ref{section:best of both worlds})

\end{itemize}

These bounds are incomparable in general. 
All of them have known counterparts in the full information setting (see for example~\citep{steinhardt2014adaptivity} and~\citep{de2014follow}),
but are novel in the bandit setting to the best of our knowledge.
Note that for the first two results that depend on some quantities of only the best arm, 
we require tuning a learning rate parameter in terms of these (unknown) quantities.
Obtaining the same results with parameter-free algorithms remains open, even for the full information setting.
However, for the other results, we indeed provide parameter-free algorithms based on a variant of the doubling trick.

Our general algorithm falls into the Online Mirror Descent (OMD) framework (see for example~\citep{hazan2016introduction})
with the ``log-barrier'' as the regularizer, originally proposed in~\citep{foster2016learning}.
However, to obtain our results, two extra crucial ingredients are needed:
\begin{itemize}
\item
First, we adopt the ideas of optimism and adaptivity from~\citep{steinhardt2014adaptivity},
which roughly speaking amounts to incorporating a correction term as well as an optimistic prediction into the loss vectors.
In~\citep{steinhardt2014adaptivity}, this technique was developed in the Follow-the-Regularized-Leader (FTRL) framework,\footnote{%
Although it was confusingly referred as OMD in~\citep{steinhardt2014adaptivity}.}
but it is in fact crucial here to re-derive it in the OMD framework (due to the next ingredient).
The challenges here are to come up with the right correction terms and optimistic predictions.

\item
Second, we apply an individual and increasing learning rate schedule for one of the path-length results.
Such increasing learning rate schedule was originally proposed in~\citep{bubeck2016kernel} and also recently used in~\citep{agarwal2017corralling},
but for different purposes.

\end{itemize}

Although most algorithmic techniques we use in this work have been studied before,
combining all of them, in the general semi-bandit setting, requires novel and non-trivial analysis.
The use of log-barrier in the semi-bandit setting is also new as far as we know.

\paragraph{Related work.}
There is a rich literature in deriving adaptive algorithms and regret bounds for online learning with full information feedback
(see recent work~\citep{luo2015achieving, koolen2015second, van2016metagrad, orabona2016coin, cutkosky2017online} and references therein),
as well as the stochastic bandit setting (such as~\citep{garivier2011kl, lattimore2015optimally, degenne2016anytime}).
Similar results for the adversarial bandit setting, however, are relatively sparse and have been mentioned above.
While obtaining regret bounds that depend on the quality of the best action is common in the full information setting,
it is in fact much more challenging in the bandit setting, 
and the only existing result of this kind is the ``small-loss'' bound~\citep{allenberg2006hannan, foster2016learning}.
We hope that our work opens up more possibilities in obtaining these results,
despite some recent negative results discovered by~\citet{gerchinovitz2016refined}.

\citet{chiang2013beating} proposed bandit algorithms with second-order path-length bounds, but their work requires stronger two-point feedback.
The implication of path-length regret bounds on faster convergence rate for computing equilibriums was studied in~\citep{syrgkanis2015fast}.
Other examples of adaptive online learning leading to faster convergence in game theory include~\citep{rakhlin2013optimization, daskalakis2015near, foster2016learning}.

There exist several bandit algorithms that achieve almost optimal regret in both the adversarial setting ($\mathcal{O}(\sqrt{TK})$) and 
the i.i.d. setting ($\mathcal{O}(\sum_{i: \Delta_i \neq 0}\frac{\ln T}{\Delta_i})$ where $\Delta_i$ is the gap between the expected loss of arm $i$ and the one of the optimal arm)~\citep{bubeck2012best, seldin2014one, auer2016algorithm, seldin2017improved}.
Our results in Section~\ref{section:best of both worlds} have slightly weaker guarantee for the i.i.d. setting (at most $K$ times worse specifically)
since it essentially replaces all $\Delta_i$ by $\min_{i: \Delta_i \neq 0} \Delta_i$.
On the other hand, however, our results have several advantages compared to previous work.
First, our guarantee for the adversarial setting is stronger since it replaces the dependence on $T$ by the loss of the best arm.
Second, our logarithmic regret result applies to not just the simple i.i.d. setting, 
but the more general setting mentioned above where neither independence nor identical distributions is required.
Our dependence on $\ln T$ is also better than previous works,
resolving an open problem raised by~\citet{seldin2017improved}.
Finally, our algorithm and analysis are also arguably much simpler, without performing any stationarity detection or gap estimation.
Indeed, the result is in some sense algorithm-independent and solely through a new adaptive regret bound Eq.~\eqref{eqn:new_excess_loss_bound},
similar to the results in the full-information setting such as~\citep{gaillard2014second}.

Using a self-concordant barrier as regularizer was proposed in the seminal work of~\citep{abernethy2008competing} for general linear bandit problems.
The log-barrier is technically not a barrier for the decision set of the semi-bandit problem, 
but still it exhibits many similar properties as shown in our proofs.
Optimistic FTRL/OMD was developed in~\citep{chiang2012online, rakhlin2013online}.
As pointed out in~\citep{steinhardt2014adaptivity}, incorporating correction terms in the loss vectors can also be viewed as using adaptive regularizers,
which was studied in several previous works, mostly for the full information setting (see~\citep{mcmahan2017survey} for a survey).

\section{Problem Setup and Algorithm Overview}
\label{section:notations}
We consider the combinatorial bandit problem with semi-bandit feedback, which subsumes the classic multi-armed bandit problem. The learning process proceeds for $T$ rounds. In each round, the learner selects a subset of arms, denoted by a binary vector $b_t$ from a predefined action set $\mathcal{X}\subseteq \{0,1\}^K$, and suffers loss $b_t^\top \ell_t$, where $\ell_t \in [-1,1]^K$ is a loss vector decided by an adversary. The feedback received by the learner is the vector $(b_{t,1}\ell_{t,1}, \ldots, b_{t,K}\ell_{t,K})$,
or in other words, the loss of each chosen arm. For simplicity, we assume that the adversary is oblivious and the loss vectors $\ell_1, \ldots, \ell_T$ are decided ahead of time independent of the learner's actions.

The learner's goal is to minimize the {\it regret}, which is the gap between her accumulated loss and that of the best fixed action $b^*\in\mathcal{X}$. Formally the regret is defined as
\begin{align*}
\reg_T\triangleq \sum_{t=1}^T b_t^\top \ell_{t}- \sum_{t=1}^T b^{*\top}\ell_{t}, \text{ where } b^*\triangleq \min_{b\in\mathcal{X}} \sum_{t=1}^T b^\top \ell_{t}. 
\end{align*}

In the special case of multi-armed bandit, the action set $\mathcal{X}$ is $\{\mathbf{e}_1, \mathbf{e}_2, \ldots, \mathbf{e}_K\}$ where $\mathbf{e}_i$ denotes the $i$-th standard basis vector. In other words, in each round the learner picks one arm $i_t \in [K] \triangleq \{1,2,\ldots,K\}$ (corresponding to $b_t = \mathbf{e}_{i_t}$), and receives the loss $\ell_{t,i_t}$. We denote the best arm by $i^* \triangleq \min_{i \in [K]} \sum_{t=1}^T \ell_{t, i}$.

\paragraph{Notation.}
For a convex function $\psi$ defined on a convex set $\Omega$, the Bregman divergence of two points $u, v\in \Omega$ with respect to $\psi$ is defined as $D_{\psi}(u,v)\triangleq\psi(u)-\psi(v)-\inn{\nabla\psi(v), u-v}$. The log-barrier used in this work is of the form $\psi(u)=\sum_{i=1}^K \frac{1}{\eta_i}\ln \frac{1}{u_i}$ for some learning rates $\eta_1, \ldots, \eta_K \geq 0$ and $u \in \text{conv}(\mathcal{X})$, the convex hull of $\mathcal{X}$. With $h(y)\triangleq y-1-\ln y$, the Bregman divergence with respect to the log-barrier is: 
$D_{\psi}(u,v)=\sum_{i=1}^K \frac{1}{\eta_i} \left(\ln \frac{v_i}{u_i} + \frac{u_i-v_i}{v_i}\right)=\sum_{i=1}^K \frac{1}{\eta_i} h\left(\frac{u_i}{v_i}\right).$


The all-zero and all-one vector are denoted by $\mathbf{0}$ and $\mathbf{1}$ respectively.
$\Delta_K$ represents the ($K-1$)-dimensional simplex.
For a binary vector $b$ we write $i\in b$ if $b_i=1$. 
Denote by $K_0 = \max_{b \in \mathcal{X}}\|b\|_0$ the maximum number of arms an action in $\mathcal{X}$ can pick.
Note that for MAB, $K_0$ is simply $1$.

We define $\ell_0 = \mathbf{0}$ for notational convenience. 
At round $t$, for an arm $i$ we denote its accumulated loss by $L_{t,i}\triangleq \sum_{s=1}^t \ell_{s,i}$,
its average loss by $\mu_{t,i} \triangleq \frac{1}{t}L_{t,i}$,
its (unnormalized) variance by $Q_{t,i}\triangleq \sum_{s=1}^t (\ell_{s,i}-\mu_{t,i})^2$,
and its first-order path-length by $V_{t,i}\triangleq \sum_{s=1}^t \absolute{\ell_{s,i}-\ell_{s-1,i}}$. 
For MAB, we define $\alpha_i(t)$ to be the most recent time when arm $i$ is picked prior to round $t$ ,
that is, $\alpha_i(t) = \max\{s < t : i_s = i\}$ (or $0$ if the set is empty).

\subsection{Algorithm Overview}
As mentioned our algorithm falls into the OMD framework that operates on the set $\Omega = \text{conv}(\mathcal{X})$.
The vanilla OMD formula for the bandit setting is $w_{t} = \argmin_{w\in\Omega} \{ \inn{w, \hat{\ell}_{t-1}}+D_{\psi}(w,w_{t-1}) \}$
for some regularizer $\psi$ and some (unbiased) estimator $\hat{\ell}_{t-1}$ of the true loss $\ell_{t-1}$.
The learner then picks an action $b_t$ randomly such that $\mathbb{E}[b_t] = w_t$, and constructs the next loss estimator $\hat{\ell}_t$ based on the bandit feedback.
Our algorithm, however, requires several extra ingredients. The generic update rule is
\begin{align}
w_t &= \argmin_{w\in\Omega} \left\{ \inn{w, m_t}+D_{\psi_t}(w,w_t^\p) \right\},\label{eqn:update_rule_1}\\ 
w_{t+1}^\p &= \argmin_{w\in\Omega} \left\{ \inn{w, \hat{\ell}_t+a_t}+D_{\psi_t}(w,w_t^\p) \right\}  \label{eqn:update_rule_2}.
\end{align}

\begin{algorithm}[t]
\DontPrintSemicolon
\caption{\small \textbf{B}arrier-\textbf{R}egularized with \textbf{O}ptimism and \textbf{AD}aptivity \textbf{O}nline \textbf{M}irror \textbf{D}escent (\textsc{\textbf{Broad-OMD}})}
\label{alg:general}
\textbf{Define}: $\Omega=\text{conv}(\mathcal{X})$, $\psi_t(w)=\sum_{i=1}^K \frac{1}{\eta_{t,i}}\ln\frac{1}{w_i}$. \\
\textbf{Initialize}: $w_1^\p = \argmin_{w\in \Omega}\psi_1(w)$.\\
\For{$t=1, 2, \ldots, T$}{
   $w_t = \argmin_{w\in\Omega} \big\{ \inner{w,m_t} + D_{\psi_t}(w, w_t^\p)\big\}$. \\
   Draw $b_t\sim w_t$, suffer loss $b_t^\top \ell_t$, and observe $\{b_{t,i}\ell_{t,i}\}_{i=1}^K$. \\
   Construct $\hat{\ell}_t$ as an unbiased estimator of $\ell_t$. \\
   Let $a_{t,i}=\begin{cases}
       6\eta_{t,i}w_{t,i}(\hat{\ell}_{t,i}-m_{t,i})^2,  &\text{(Option I)}\\
       0. &\text{(Option II)}
       \end{cases}$\\
   $w_{t+1}^\p=\argmin_{w\in\Omega} \big\{ \langle w,\hat{\ell}_t+ a_t\rangle  +D_{\psi_t}(w, w^\p_t) \big\}.$ 
}    
\end{algorithm}

 \renewcommand{\arraystretch}{1.4}
\begin{table}[t] 
  \centering
  \caption{Different configurations of \textsc{Broad-OMD} and regret bounds for MAB. 
  See Section~\ref{section:notations} and the corresponding sections for the meaning of notation. 
  For the last two rows, to obtain parameter-free algorithms one needs to apply a doubling trick to decrease the learning rate.}
  \begin{threeparttable}
  \begin{tabular}{ | c | c | c | c | c | c | }
    \hline
    Sec. & Option & $m_{t,i}$ & $\hat{\ell}_{t,i}$ & $\eta_{t,i}$ & $\mathbb{E}[\reg_T]$ in $\tilde{\mathcal{O}}$ \\ \hline
    \ref{subsubsection:variation bound} & I & $\tilde{\mu}_{t-1,i}$ & $\frac{(\ell_{t,i}-m_{t,i})\mathbbm{1}\{i_t=i\}}{w_{t,i}}+m_{t,i}$ & fixed & $\sqrt{KQ_{T,i^*}}$ \\ \hline
    \ref{subsubsection:path-length} &  I & $\ell_{\alpha_i(t),i}$ & $\frac{(\ell_{t,i}-m_{t,i})\mathbbm{1}\{i_t=i\}}{\bar{w}_{t,i}}+m_{t,i}$ & increasing & $K\sqrt{V_{T,i^*}}$ \\ \hline
    \ref{subsection:first_order_better_k} & II & $\ell_{\alpha_i(t),i}$ & $\frac{(\ell_{t,i}-m_{t,i})\mathbbm{1}\{i_t=i\}}{w_{t,i}}+m_{t,i}$ & fixed & $ \sqrt{K\sum_{i=1}^K V_{T,i}}$ \\  \hline
   \ref{section:best of both worlds} & II & $\ell_{t,i_t}$ &  $\frac{\ell_{t,i}\mathbbm{1}\{i_t=i\}}{w_{t,i}}$ & fixed & $\min\{ \sqrt{KL_{T,i^*}}, \frac{K}{\Delta}\}$ \\ \hline
  \end{tabular}
  \end{threeparttable}
  \label{table:summary}
\end{table}

Here, we still play randomly according to $w_t$, which is now updated to minimize its loss with respect to $m_t \in [-1,1]^K$, 
an {\it optimistic prediction} of the true loss vector $\ell_t$,
penalized by a Bregman divergence term associated with a {\it time-varying} regularizer $\psi_t$.
In addition, we maintain a sequence of auxiliary points $w_t^\p$ that is updated using the loss estimator $\hat{\ell}_t$ and an extra {\it correction term} $a_t$.

When $a_t = \mathbf{0}$, this is studied in~\citep{rakhlin2013online} under the name optimistic OMD. 
When $a_t \neq \mathbf{0}$, the closest algorithm to this variant of OMD is its FTRL version studied by~\citet{steinhardt2014adaptivity}.
However, while $\psi_t$ is fixed for all $t$ in~\citep{steinhardt2014adaptivity},\footnote{%
\citet{steinhardt2014adaptivity} also uses the notation $\psi_t$, but it corresponds to putting $a_t$ into a fixed regularizer.}
some of our results crucially rely on using time-varying $\psi_t$ (which corresponds to time-varying learning rate)
and also the OMD update form instead of FTRL. 

It is well known that the classic Exp3 algorithm falls into this framework with $m_t = a_t = \mathbf{0}$ and $\psi_t$ being the (negative) entropy.
To obtain our results, first, it is crucial to use the log-barrier as the regularizer instead, that is, $\psi_t(w)=\sum_{i=1}^K \frac{1}{\eta_{t,i}}\ln\frac{1}{w_i}$
for some individual and time-varying learning rates $\eta_{t,i}$.
Second, we focus on two options of $a_t$.
For results that depend on some quantity of only the best arm, we use a sophisticated choice of $a_t$ that we explain in details in Section~\ref{section:Option I}.
For the other results we simply set $a_t = \mathbf{0}$.
With the choices of $m_t, \hat{\ell}_t$, and $\eta_{t}$ open, we present this generic framework in Algorithm~\ref{alg:general}
and name it \textsc{Broad-OMD} (short for Barrier-Regularized with Optimism and ADaptivity Online Mirror Descent).

In Section~\ref{section:Option I} and~\ref{section:Option II} respectively, we prove general regret bounds for \textsc{Broad-OMD} 
with Option I and Option II, followed by specific applications in the MAB setting achieved via specific choices of $m_{t}, \hat{\ell}_t$, and $\eta_{t}$.
The results and the corresponding configurations of the algorithm are summarized in Table \ref{table:summary}.

\paragraph{Computational efficiency.} 
The sampling step $b_t \sim w_t$ can be done efficiently as long as $\Omega$ can be described by a polynomial number of constraints.
The optimization problems in the update rules of $w_t$ and $w_t'$ are convex and can be solved by general optimization methods.
For many special cases, however, these two computational bottlenecks have simple solutions.
Take MAB as an example, $w_t$ directly specifies the probability of picking each arm,
and the optimization problems can be solved via a simple binary search~\citep{agarwal2017corralling}.

\section{\textsc{Broad-OMD} with Option I}
\label{section:Option I}
In this section we focus on \textsc{Broad-OMD} with Option I.
We first show a general lemma that update rules~\eqref{eqn:update_rule_1} and~\eqref{eqn:update_rule_2} guarantee,
no matter what regularizer $\psi_t$ is used and what $a_t, m_t$, and $\hat{\ell}_t$ are.

\begin{lemma}
\label{thm:general_instantaneous}
For the update rules~\eqref{eqn:update_rule_1} and~\eqref{eqn:update_rule_2}, if the following condition holds:
\begin{align}
\inn{w_t-w^\p_{t+1}, \hat{\ell}_t-m_t+a_t} \leq \inn{w_t, a_t}, \label{eqn:condition1} 
\end{align}
then for all $u\in \Omega$, we have
\begin{align}
\inn{w_t-u, \hat{\ell}_t}\leq D_{\psi_t}(u,w_t^\p)-D_{\psi_t}(u,w^\p_{t+1})+\inn{u,a_t}-A_t,\label{eqn:regret_bound:bregman}
\end{align}
where $A_t\triangleq D_{\psi_t}(w_{t+1}^\p, w_t)+D_{\psi_t}(w_t, w_t^\p)\geq 0$.
\end{lemma}
 




The important part of bound~\eqref{eqn:regret_bound:bregman} is the term $\inn{u,a_t}$, 
which allows us to derive regret bounds that depend on only the comparator $u$.
The key is now how to configure the algorithm such that condition~\eqref{eqn:condition1} holds, 
while leading to a reasonable bound~\eqref{eqn:regret_bound:bregman} at the same time. 

In the work of~\citep{steinhardt2014adaptivity} for full-information problems, $a_t$ can be defined as $a_{t,i} = \eta_{t,i}(\ell_{t,i}-m_{t,i})^2$,
which suffices to derive many interesting results.
However, in the bandit setting this is not applicable since $\ell_t$ is unknown.
The natural first attempt is to replace $\ell_t$ by $\hat{\ell}_t$, but one would quickly realize the common issue in the bandit literature:
$\hat{\ell}_{t,i}$ is often constructed via inverse propensity weighting, and thus $(\hat{\ell}_{t,i}-m_{t,i})^2$ can be of order $1/w_{t,i}^2$, which is too large.

Based on this observation, our choice for $a_t$ is $a_{t,i}=6\eta_{t,i}w_{t,i}(\hat{\ell}_{t,i}-m_{t,i})^2$ (the constant $6$ is merely for technical reasons). 
The extra term $w_{t,i}$ can then cancel the aforementioned large term $1/w_{t,i}^2$ in expectation, similar to the classic trick done in the analysis of Exp3~\citep{auer2002nonstochastic}.



Note that with a smaller $a_t$, condition~\eqref{eqn:condition1} becomes more stringent.
The entropy regularizer used in~\citep{steinhardt2014adaptivity} no longer suffices to maintain such a condition.
Instead, it turns out that the log-barrier regularizer used by \textsc{Broad-OMD} addresses the issue, as shown below.

\begin{theorem}
\label{lemma:MAB_condition}
If the following three conditions hold for all $t,i$: 
(i) $\eta_{t,i}\leq \frac{1}{162}$,
(ii) $w_{t,i}\abs{\hat{\ell}_{t,i}-m_{t,i}}\leq 3$,
(iii) $\sum_{i=1}^K \eta_{t,i}w_{t,i}^2(\hat{\ell}_{t,i}-m_{t,i})^2\leq \frac{1}{18},$
then \textsc{Broad-OMD} with $a_{t,i}=6\eta_{t,i}w_{t,i}(\hat{\ell}_{t,i}-m_{t,i})^2$ guarantees condition~\eqref{eqn:condition1}.
Moreover, it guarantees for any $u\in \Omega$ (recall $h(y) = y - 1 - \ln y \geq 0$), 
\begin{align}
\sum_{t=1}^T \inn{w_t-u, \hat{\ell}_t}\leq \sum_{i=1}^K \left( \frac{\ln\frac{w^\p_{1,i}}{u_i}}{\eta_{1,i}} + \sum_{t=1}^T \left(\frac{1}{\eta_{t+1,i}}-\frac{1}{\eta_{t,i}}\right)h\left(\frac{u_i}{w_{t+1,i}^\p}\right) \right) +\sum_{t=1}^T \inn{u,a_t}. \label{eqn:regret_bound:a_t_neq_0} 
\end{align}
\end{theorem}

The three conditions of the theorem are usually trivially satisfied as we will show.
Note that $h(\cdot)$ is always non-negative. Therefore, if the sequence $\{\eta_{t,i}\}_{t=1}^{T+1}$ is non-decreasing for all $i$,\footnote{%
One might notice that $\eta_{T+1, i}$ is not defined here.
Indeed this term is artificially added only to make the analysis of Section~\ref{subsubsection:path-length} more concise, and $\eta_{T+1,i}$ can be any positive number.
In Algorithm~\ref{alg:increasing} we give it a concrete definition.
} 
the term $\sum_{t=1}^T \left(\frac{1}{\eta_{t+1,i}}-\frac{1}{\eta_{t,i}}\right)h\left(\frac{u_i}{w_{t+1,i}^\p}\right)$ in bound~\eqref{eqn:regret_bound:a_t_neq_0} 
is non-positive. For some results we can simply discard this term, while for others, this term becomes critical.
On the other hand, the term $\ln\frac{w^\p_{1,i}}{u_i}$ appears to be infinity if we want to compare with the best fixed action (where $u_i = 0$ for some $i$).
However, this can be simply resolved by comparing with some close neighbor of the best action in $\Omega$ instead, similar to~\citep{foster2016learning, agarwal2017corralling}.

One can now derive different results using Theorem~\ref{lemma:MAB_condition} with specific choices of $\hat{\ell}_t$ and $m_t$.
As an example, we state the following corollary by using a variance-reduced importance-weighted estimator $\hat{\ell}_t$ as in~\citep{rakhlin2013online}.
\begin{cor}
\label{cor:clear_corollary}
\textsc{Broad-OMD} with $a_{t,i}=6\eta_{t,i}w_{t,i}(\hat{\ell}_{t,i}-m_{t,i})^2$, any $m_{t,i} \in [-1, 1]$, $\hat{\ell}_{t,i}=\frac{(\ell_{t,i}-m_{t,i})\mathbbm{1}\{i\in b_t\}}{w_{t,i}}+m_{t,i}$, and \sloppy $\eta_{t,i}=\eta \leq \frac{1}{162K_0}$ enjoys the following regret bound:
\begin{align*}
\mathbb{E}\left[\reg_T\right]= \mathbb{E}\left[ \sum_{t=1}^T \inn{b_t-b^*, \ell_t} \right] \leq \frac{K\ln T}{\eta} + 6\eta\mathbb{E}\left[\sum_{t=1}^T \sum_{i: i\in b^*} (\ell_{t,i}-m_{t,i})^2\right] +\mathcal{O}(K_0). 
\end{align*}
\end{cor}

One can see that the expected regret in Corollary \ref{cor:clear_corollary} only depends on the squared estimation error of $m_t$ for the actions that $b^*$ chooses! This is exactly the counterpart of results in~\citep{steinhardt2014adaptivity}, but for the more challenging combinatorial semi-bandit problem. 
Note that our dependence on $K_0$ is also optimal~\citep{audibert2013regret}.

In the following subsections, we invoke Theorem \ref{lemma:MAB_condition} with different choices of $\hat{\ell}_{t}$ and $m_t$ to obtain various more concrete adaptive bounds. 
For simplicity, we state these results only in the MAB setting, but they can be straightforwardly generalized to the semi-bandit case.

\subsection{Variance Bound}
\label{subsubsection:variation bound}
Our first application of \textsc{Broad-OMD} is an adaptive bound that depends on the variance of the best arm, that is, a bound of order $\tilde{\mathcal{O}}\left(\sqrt{KQ_{T,i^*}}\right)=\tilde{\mathcal{O}}\left(\sqrt{K\sum_{t=1}^T(\ell_{t,i^*}-\mu_{T,i^*})^2}\right)$.  
According to Corollary~\ref{cor:clear_corollary}, if we were able to use $m_t = \mu_T$, with a best-tuned $\eta$ the bound is obtained immediately.
The issue is of course that $\mu_T$ is unknown ahead of time. 
In fact, even setting $m_t = \mu_{t-1}$ is infeasible due to the bandit feedback.

Fortunately this issue was already solved by~\citet{hazan2011better} via the ``reservoir sampling'' technique. 
The high level idea is that one can spend a small portion of time on estimating $\mu_{t}$ on the fly. More precisely, by performing uniform exploration with probability $\min\left\{1, \frac{MK}{t}\right\}$ at time $t$ for some parameter $M$, one can obtain an estimator $\tilde{\mu}_{t}$ of $\mu_t$  such that $\mathbb{E}[\tilde{\mu}_{t}] = \mu_{t}$ and $\text{Var}[\tilde{\mu}_{t,i}]\leq \frac{Q_{t,i}}{Mt}$ (see~\citep{hazan2011better} for details). 
Then we can simply pick $m_t=\tilde{\mu}_{t-1}$ 
and prove the following result.
\begin{theorem}
\label{cor:variance_bound}
\textsc{Broad-OMD} with reservoir sampling~\citep{hazan2011better}, $a_{t,i}=6\eta_{t,i}w_{t,i}(\hat{\ell}_{t,i}-m_{t,i})^2$, $m_{t,i}=\tilde{\mu}_{t-1,i}$, $\hat{\ell}_{t,i}=\frac{(\ell_{t,i}-m_{t,i})\mathbbm{1}\{i_t=i\}}{w_{t,i}}+m_{t,i}$, and $\eta_{t,i}=\eta\leq \frac{1}{162}$ guarantees
\begin{align*}
\mathbb{E}\left[\reg_T\right]= \mathcal{O}\left(\frac{K\ln T}{\eta}+\eta Q_{T,i^*} + K(\ln T)^2\right).
\end{align*}
With the optimal tuning of $\eta$, the regret is thus of order $\tilde{\mathcal{O}}\left(\sqrt{KQ_{T,i^*}}+K\right)$.
\end{theorem}

\subsection{Path-length Bound}
\label{subsubsection:path-length}

Our second application is to obtain path-length bounds.
The counterpart in the full-information setting is a bound in terms of the second-order path-length $\sum_{t=1}^T(\ell_{t,i^*}-\ell_{t-1,i^*})^2$~\citep{steinhardt2014adaptivity}. 
Again, in light of Corollary~\ref{cor:clear_corollary}, if we were able to pick $m_t = \ell_{t-1}$ the problem would be solved.
The difficulty is again that $\ell_{t-1}$ is not fully observable.


While it is still not clear how to achieve such a second-order path-length bound or whether it is possible at all,
we propose a way to obtain a slightly weaker first-order path-length bound 
$\tilde{\mathcal{O}}\left(K\sqrt{V_{T,i^*}}\right)=\tilde{\mathcal{O}}\Big(K\sqrt{\sum_{t=1}^T\abs{\ell_{t,i^*}-\ell_{t-1,i^*}}}\Big)$.
Note that in the worst case this is $\sqrt{K}$ times worse than the optimal regret $\tilde{\mathcal{O}}(\sqrt{TK})$.


The idea is to set $m_{t,i}$ to be the most recent observed loss of arm $i$, that is, $m_{t,i}=\ell_{\alpha_i(t),i}$, where $\alpha_i(t)$ is defined in Section \ref{section:notations}.
While the estimation error $(\ell_{t,i}-\ell_{\alpha_i(t),i})^2$ could be much larger than $(\ell_{t,i}-\ell_{t-1,i})^2$, the quantity we aim for, 
observe that 
if $t-\alpha_i(t)$ is large, it means that arm $i$ has bad performance before time $t$ so that the learner seldom draws arm $i$.
In this case, the learner might have accumulated negative regret with respect to arm $i$, which can potentially be used to compensate the large estimation error. 

To formalize this intuition, we go back to the bound in Theorem~\ref{lemma:MAB_condition} and examine the key term $\sum_{t=1}^T \inn{u, a_t}$
after plugging in $u = \mathbf{e}_i$ for some arm $i$, $m_{t,i}=\ell_{\alpha_i(t),i}$, and $\hat{\ell}_{t,i}=\frac{(\ell_{t,i}-m_{t,i})\mathbbm{1}\{i_t=i\}}{w_{t,i}}+m_{t,i}$. 
We assume $\eta_{t,i}=\eta$ for simplicity and also use the fact $w_{t,i}\abs{\hat{\ell}_{t,i}-m_{t,i}}\leq 2$.
We then have
\begin{align}
\sum_{t=1}^T \inn{u,a_t}&=6\eta\sum_{t=1}^T w_{t,i}(\hat{\ell}_{t,i}-\ell_{\alpha_i(t),i})^2 \leq 12\eta\sum_{t=1}^T \abs{\hat{\ell}_{t,i}-\ell_{\alpha_i(t),i}}
= 12\eta\sum_{t: i_t=i} \frac{\abs{\ell_{t,i}-\ell_{\alpha_i(t),i}} }{w_{t,i}} \nonumber \\
&\leq 12\eta\sum_{t: i_t=i} \frac{\sum_{s=\alpha_i(t)+1}^t \abs{\ell_{s,i}-\ell_{s-1,i}} }{w_{t,i}}  
\leq12\eta \left(\max_{t\in[T]} \frac{1}{w_{t,i}}\right) V_{T,i}. \label{eqn:path_length_trick}
\end{align}


Therefore, the term $\sum_{t=1}^T \inn{u,a_t}$ is close to the first-order path-length but with an extra factor $\max_{t\in[T]} \frac{1}{w_{t,i}}$.
To cancel this potentially large factor, we adopt the increasing learning rate schedule recently used in~\citep{agarwal2017corralling}. 
The idea is that the term $h\big(\frac{u_i}{w_{t+1,i}^\p}\big)$ in Eq.~\eqref{eqn:regret_bound:a_t_neq_0} is close to $\frac{1}{w_{t+1,i}}$ if $u_i$ is close to $1$.
If we increase the learning rate whenever we encounter a large $\frac{1}{w_{t+1,i}}$, 
then $\Big(\frac{1}{\eta_{t+1,i}}-\frac{1}{\eta_{t,i}}\Big)h\Big(\frac{u_i}{w_{t+1,i}^\p}\Big)$ becomes a large negative term in terms of $\frac{-1}{w_{t+1,i}}$,
which exactly compensates the term $\sum_{t=1}^T \inn{u,a_t}$.

To avoid the learning rates increased by too much, 
similarly to~\citep{agarwal2017corralling} we use some individual threshold ($\rho_{t,i}$) to decide when to increase the learning rate
and update these thresholds in some doubling manner. 
Also, we mix $w_t$ with a small amount of uniform exploration to further ensure that it cannot be too small.
The final algorithm, call \textsc{Broad-OMD+}, is presented in Algorithm~\ref{alg:increasing} (only for the MAB setting for simplicity).
We prove the following theorem.

\begin{algorithm}[t]
\DontPrintSemicolon
\caption{\textsc{Broad-OMD}+ (specialized for MAB)}
\label{alg:increasing}
\textbf{Define:} $\kappa=e^{\frac{1}{\ln T}}$, $\psi_t(w)= \sum_{i=1}^K \frac{1}{\eta_{t,i}} \ln \frac{1}{w_{i}}$. \\
\textbf{Initialize}: $w^\p_{1, i} = 1/K$, $\rho_{1, i} = 2K$ for all $i \in [K]$.\\
\For{$t=1, 2, \ldots, T$}{
   $w_t = \argmin_{w\in\Delta_K} \big\{ \inner{w,m_t} + D_{\psi_t}(w, w_t^\p)\big\}$. \\
   $\bar{w}_{t} = (1-\frac{1}{T})w_{t} + \frac{1}{KT}\mathbf{1}$. \\
   Draw $i_t\sim \bar{w}_t$, suffer loss $\ell_{t,i_t}$, and let $\hat{\ell}_{t,i}=\frac{(\ell_{t,i}-m_{t,i})\mathbbm{1}\{i_t=i\}}{\bar{w}_{t,i}}+m_{t,i}$.\\
   Let $a_{t,i}=6\eta_{t,i}w_{t,i}(\hat{\ell}_{t,i}-m_{t,i})^2$.
   \\
   $w_{t+1}^\p=\argmin_{w\in\Delta_K} \big\{ \langle w,\hat{\ell}_t+a_t\rangle  +D_{\psi_t}(w, w^\p_t) \big\}.$ \\
   \For{$i=1, \ldots, K$}{
      \lIf{$\frac{1}{\bar{w}_{t,i}} > \rho_{t,i}$}{
         $\rho_{t+1,i}=\frac{2}{\bar{w}_{t,i}}$, $\eta_{t+1,i}=\kappa\eta_{t,i}$. 
      }
      \lElse{
         $\rho_{t+1,i}=\rho_{t,i}$, $\eta_{t+1,i}=\eta_{t,i}$.
      }
   }
}    
\end{algorithm}

\begin{theorem}
\label{thm:path_length}
\textsc{Broad-OMD+} with $m_{t,i}=\ell_{\alpha_i(t), i}$ and $\eta_{1,i}=\eta\leq \frac{1}{810}$ guarantees 
\begin{align*}
\mathbb{E}\left[\reg_T \right] \leq \frac{2K\ln T}{\eta} + \mathbb{E}[\rho_{T+1,i^*}]\left( \frac{-1}{40\eta\ln T} + 90\eta V_{T,i^*} \right) + \mathcal{O}\left( 1 \right)
\end{align*}
when $T\geq 3$. Picking $\eta = \min\Big\{\frac{1}{810}, \frac{1}{60\sqrt{V_{T,i^*} \ln T}}\Big\}$ so that the second term is non-positive leads to $\mathbb{E}\left[\reg_T \right] = \tilde{\mathcal{O}}\left( K\sqrt{V_{T,i^*}}+K \right)$. 
\end{theorem}

\section{\textsc{Broad-OMD} with Option II}
\label{section:Option II}

In this section, we move on to discuss \textsc{Broad-OMD} with Option II, that is, $a_t = \mathbf{0}$. 
We also fix $\eta_{t,i} = \eta$, although in the doubling trick discussed later, different values of $\eta$ will be used for different runs of  \textsc{Broad-OMD}.
Again we start with a general lemma that holds no matter what regularizer $\psi_t$ is used and what $m_t$ and $\hat{\ell}_t$ are.

\begin{lemma}
\label{lemma:simple_lemma}
For the update rules~\eqref{eqn:update_rule_1} and~\eqref{eqn:update_rule_2} with $a_t=\mathbf{0}$, we have for all $u\in \Omega$, 
\begin{align*}
\inn{w_t-u, \hat{\ell}_t}\leq D_{\psi_t}(u,w_t^\p)-D_{\psi_t}(u,w^\p_{t+1})+\inn{w_t-w_{t+1}^\p, \hat{\ell}_t-m_t}-A_t, 
\end{align*}
where $A_t\triangleq D_{\psi_t}(w_{t+1}^\p, w_t)+D_{\psi_t}(w_t, w_t^\p)\geq 0$.
\end{lemma}

The proof is standard as in typical OMD analysis. The next theorem then shows how the term $\inn{w_t-w_{t+1}^\p, \hat{\ell}_t-m_t}$ is further bounded
when $\psi_t$ is the log-barrier as in \textsc{Broad-OMD}. 

\begin{theorem}
\label{lemma:second_order_regret_bound}
If the following three conditions hold for all $t,i$: 
(i) $\eta \leq \frac{1}{162}$,
(ii) $w_{t,i}\abs{\hat{\ell}_{t,i}-m_{t,i}}\leq 3$,
(iii) $\eta\sum_{i=1}^K w_{t,i}^2(\hat{\ell}_{t,i}-m_{t,i})^2\leq \frac{1}{18}$
(same as those in Theorem \ref{lemma:MAB_condition}), 
then \textsc{Broad-OMD} with $a_t=\mathbf{0}$ guarantees for any $u\in \Omega$, 
\begin{align}
\sum_{t=1}^T \inn{w_t-u, \hat{\ell}_t}\leq \sum_{i=1}^K  \frac{\ln \frac{w^\p_{1,i}}{u_i}}{\eta}  +3\eta\sum_{t=1}^T\sum_{i=1}^K w_{t,i}^2(\hat{\ell}_{t,i}-m_{t,i})^2-\sum_{t=1}^T A_t. \label{eqn:second_order_regret_bound}
\end{align}
For MAB, the last term can further be lower bounded by $\sum_{t=1}^T A_t \geq \frac{1}{48\eta}\sum_{t=2}^T \sum_{i=1}^K\frac{(w_{t,i}-w_{t-1,i})^2}{w_{t-1,i}^2}$.
\end{theorem}

In bound~\eqref{eqn:second_order_regret_bound}, 
the first term can again be bounded by $\frac{K \ln T}{\eta}$ via picking an appropriate $u$.
The last negative term is useful when we use the algorithm to play games, which is discussed in Section~\ref{subsection:games}.
The second term is the key term, which, compared to the key term $\sum_{t=1}^T \inn{u,a_t}$ in Eq.~\eqref{eqn:regret_bound:a_t_neq_0} for \textsc{Broad-OMD} with Option I,
has an extra $w_{t,i}$ and is in terms of all arms instead of the arms that $u$ picks.
As a comparison to Corollary~\ref{cor:clear_corollary}, if we pick $\hat{\ell}_{t,i}=\frac{(\ell_{t,i}-m_{t,i})\mathbbm{1}\{i\in b_t\}}{w_{t,i}}+m_{t,i}$,
we obtain an expected regret bound in terms of $\mathbb{E}\left[ \sum_{t=1}^T \sum_{i \in b_t} (\ell_{t,i} - m_{t,i})^2 \right] = 
\mathbb{E}\left[ \sum_{t=1}^T \sum_{i =1}^K w_{t,i} (\ell_{t,i} - m_{t,i})^2 \right]$,
which is not as easy to interpret as the bound in Corollary~\ref{cor:clear_corollary}.
However, in the following subsections we will discuss in details how to apply bound~\eqref{eqn:second_order_regret_bound} to obtain more concrete results.


Before that, we point out that since the bound is now in terms of all arms, 
we can in fact apply a doubling trick to make the algorithm parameter-free!
The idea is that 
as long as the observable term $3\eta\sum_{s=1}^t \sum_{i=1}^K w_{s,i}^2(\hat{\ell}_{s,i}-m_{s,i})^2$ becomes larger than $\frac{K\ln T}{\eta}$ at some round $t$, 
we half the learning rate $\eta$ and restart the algorithm. 
This avoids the need for optimal tuning done in Section~\ref{section:Option I}.
We formally present the algorithm in Algorithm \ref{alg:doubling} (in Appendix~\ref{app:doubling_trick}) and show its regret bound below.

\begin{theorem}
\label{thm:doubling_trick_theorem}
If conditions (ii) and (iii) in Theorem~\ref{lemma:second_order_regret_bound} hold, then Algorithm \ref{alg:doubling} guarantees
\begin{align*}
\mathbb{E}[\reg_T]=\mathcal{O}\left(\sqrt{(K\ln T)\mathbb{E}\left[\sum_{t=1}^T\sum_{i=1}^Kw_{t,i}^2(\hat{\ell}_{t,i}-m_{t,i})^2\right]}+K_0K\ln T\right).
\end{align*}
\end{theorem}
In the following subsections, we instantiate Theorem~\ref{lemma:second_order_regret_bound} or~\ref{thm:doubling_trick_theorem} with different $m_{t}$ and $\hat{\ell}_t$. Again, for simplicity we only focus on the MAB setting. 

\subsection{Another Path-length Bound}
\label{subsection:first_order_better_k}
If we configure \textsc{Broad-OMD} with Option II in the same way as in Section~\ref{subsubsection:path-length},
that is, $m_{t,i}=\ell_{\alpha_i(t),i}$ and $\hat{\ell}_{t,i}=\frac{(\ell_{t,i}-m_{t,i})\mathbbm{1}\{i_t=i\}}{w_{t,i}}+m_{t,i}$.
Then the key term in Eq.~\eqref{eqn:second_order_regret_bound} can be bounded as follows:
\begin{align}
&\sum_{t=1}^T \sum_{i=1}^K  w_{t,i}^2(\hat{\ell}_{t,i}-m_{t,i})^2= \sum_{t=1}^T\sum_{i=1}^K (\ell_{t,i}-\ell_{\alpha_i(t),i})^2\mathbbm{1}\{i_t=i\} 
= \sum_{i=1}^K \sum_{t:i_t=i} (\ell_{t,i}-\ell_{\alpha_i(t),i})^2 \nonumber \\
&\leq 2 \sum_{i=1}^K \sum_{t:i_t=i} \abs{\ell_{t,i}-\ell_{\alpha_i(t),i}} 
\leq 2 \sum_{i=1}^K \sum_{t:i_t=i} \sum_{s=\alpha_i(t)+1}^t\abs{\ell_{s,i}-\ell_{s-1,i}} \leq 2 \sum_{i=1}^K V_{T,i}. \label{eqn:path_length_calculation_1}
\end{align}
Unlike Eq.~\eqref{eqn:path_length_trick}, this is bounded even without the help of negative regret, but the price is that now the regret depends on the sum of all arms' path-length. With this calculation, we obtain the following corollary.
\begin{cor}
\label{cor:path_length_bound_1}
\textsc{Broad-OMD} with $a_{t,i}=0$, $m_{t,i}=\ell_{\alpha_i(t),i}$, $\hat{\ell}_{t,i}=\frac{(\ell_{t,i}-m_{t,i})\mathbbm{1}\{i_t=i\}}{w_{t,i}}+m_{t,i}$, and $\eta_{t,i}=\eta\leq \frac{1}{162}$ guarantees 
\begin{align*}
\mathbb{E}\left[\reg_T\right]\leq\mathcal{O}\left(\frac{K\ln T}{\eta}\right) + 6\eta\sum_{i=1}^K V_{T,i} -\mathbb{E}\left[\sum_{t=2}^{T} \sum_{i=1}^K\frac{(w_{t,i}-w_{t-1,i})^2}{48\eta w_{t-1,i}^2} \right]\leq \mathcal{O}\left( \frac{K\ln T}{\eta} + \eta\sum_{i=1}^K V_{T,i}  \right). 
\end{align*}
Using the doubling trick (Algorithm~\ref{alg:doubling}), we achieve expected regret $\tilde{\mathcal{O}}\left(\sqrt{K\sum_{i=1}^K V_{T,i}} + K\right)$.
\end{cor}

This new path-length bound could be $\sqrt{K}$ times better than the one in Section~\ref{subsubsection:path-length} in some cases,
but $\sqrt{T}$ times larger in others.
The extra advantage, however, is the negative term in the regret,\footnote{%
In fact, similar negative term, coming from the term $A_t$ in Lemma~\ref{thm:general_instantaneous}, also exists (but is omitted) in the bound of Theorem~\ref{thm:path_length}.
However, it is not clear to us how to utilize it in the same way as in Section~\ref{subsection:games} if we also want to exploit the other negative term coming from increasing learning rates.
} 
explicitly spelled out in Corollary~\ref{cor:path_length_bound_1},
which we discuss next.

\subsubsection{Fast convergence in bandit games}
\label{subsection:games}

It is well-known that in a repeated two-player zero-sum game, 
if both players play according to some no-regret algorithms,
then their average strategies converge to a Nash equilibrium~\citep{freund1999adaptive}.
Similar results for general multi-player games have also been discovered.
The convergence rate of these results is governed by the regret bounds of the learning algorithms,
and several recent works (such as those mentioned in the introduction) have developed adaptive algorithms with regret much smaller than the worst case $\mathcal{O}(\sqrt{T})$ 
by exploiting the special structure in this setup,
which translates to convergence rates faster than $1/\sqrt{T}$ in computing equilibriums.



One way to obtain such fast rates is exactly via path-length regret bounds as shown in~\citep{rakhlin2013optimization, syrgkanis2015fast}. In these works, the convergence rate $1/T$ is achieved when the players have full-information feedback. 
We generalize their results to the bandit setting, and show that convergence rate of $ 1/T^{\frac{3}{4}} $ can be obtained. Though faster than $1/\sqrt{T}$, it is still slower than $1/T$ compared to the full-information setting, which is due to the fact that in bandit we only have first-order instead of second-order path-length bound. We detail the proofs and the remaining open problems in Appendix~\ref{appendix:game}. 

\subsection{Adapting to Stochastic Bandits}
\label{section:best of both worlds}
Our last application is to obtain an algorithm that simultaneously enjoys near optimal regret in both adversarial and stochastic setting. 
Specifically, the stochastic setting we consider here is as follows: there exists an arm $a^*$ and some fixed gap $\Delta > 0$ such that 
$\mathbb{E}_{\ell_t}\left[\ell_{t,i}-\ell_{t,a^*} | \ell_1, \ldots, \ell_{t-1}\right]\geq \Delta$ for all $i \neq a^*$ and $t\in[T]$.
In other words, arm $a^*$'s expected loss is always smaller than those of other arms by a fixed amount.
The classic i.i.d. MAB~\citep{lai1985asymptotically} is clearly a special case of ours.
Unlike the i.i.d. setting, however, we require neither independence nor identical distributions.

Note that $a^*$ can be different from the empirically best arm $i^*$ defined in Section~\ref{section:notations}. 
The expected regret in this setting is still with respect to $i^*$ and further takes into consideration the randomness over losses. 
In other words, we care about $\mathbb{E}_{\ell_1, \ldots, \ell_T}\left[\mathbb{E}_{i_1, \ldots, i_T}[\text{Reg}_T]\right]$, abbreviated as $\mathbb{E}[\text{Reg}_T]$ still.  

We invoke \textsc{Broad-OMD} with $a_t=\mathbf{0}$, $\hat{\ell}_{t,i}=\frac{\ell_{t,i}\mathbbm{1}\{i_t=i\}}{w_{t,i}}$ being the typical importance-weighted unbiased estimator,
and a somewhat special choice of $m_{t}$: $m_{t,i}=\ell_{t,i_t}$ for all $i$. 
This choice of $m_t$ is seemingly invalid since it depends on $i_t$, which is drawn after we have constructed $w_t$ based on $m_t$ itself.
However, note that because $m_t$ now has identical coordinates, we have
$w_t = \argmin_{w\in\Delta_K} \big\{ \inner{w,m_t} + D_{\psi_t}(w, w_t^\p)\big\} = \argmin_{w\in\Delta_K} \big\{D_{\psi_t}(w, w_t^\p)\big\} = w_t'$, independent of the actual value of $m_t$.
Therefore, the algorithm is still valid and is in fact equivalent to the vanilla log-barrier OMD of~\citep{foster2016learning}.
Also note that we cannot define $\hat{\ell}_{t}$ as in previous sections (in terms of $m_t$) since it is not an unbiased estimator of $\ell_t$ anymore (due to the randomness of $m_t$).


Although the algorithm is the same, using our analysis framework we actually derive a tighter bound in terms of the following quantity based on
Theorem~\ref{lemma:second_order_regret_bound}:
$\sum_{t=1}^T\sum_{i=1}^K w_{t,i}^2(\hat{\ell}_{t,i}-\ell_{t,i_t})^2=\sum_{t=1}^T\sum_{i=1}^K (\ell_{t,i}\mathbbm{1}\{i_t=i\}-w_{t,i}\ell_{t,i_t})^2$.
It turns out that based on this quantity alone, one can derive both a ``small-loss'' bound for the adversarial setting and a logarithmic bound for the stochastic setting
as shown below.
We emphasize that the doubling trick of Algorithm~\ref{alg:doubling} is essential to make the algorithm parameter-free,
which is another key difference from~\citep{foster2016learning}.


\begin{theorem}
\label{thm:best of both}
\textsc{Broad-OMD} with $a_t = 0$, $m_{t,i}=\ell_{t,i_t}$, $\hat{\ell}_{t,i}=\frac{\ell_{t,i}\mathbbm{1}\{i_t=i\}}{w_{t,i}}$, and
the doubling trick (Algorithm~\ref{alg:doubling}), guarantees 
\begin{equation}\label{eqn:new_excess_loss_bound}
\mathbb{E}\left[\reg_T\right]=\mathcal{O}\left(\sqrt{(K\ln T)\mathbb{E}\left[ \sum_{t=1}^T\sum_{i=1}^K (\ell_{t,i}\mathbbm{1}\{i_t=i\}-w_{t,i}\ell_{t,i_t})^2 \right]} + K\ln T \right).
\end{equation}
This bound implies that in the stochastic setting, we have $\mathbb{E}\left[\reg_T\right] = \mathcal{O}\left(\frac{K\ln T}{\Delta}\right)$, while in the adversarial setting, we have 
$\mathbb{E}\left[\reg_T\right] = \mathcal{O}\left(\sqrt{KL_{T,i^*}\ln T}+K\ln T\right)$ assuming non-negative losses.

\end{theorem}

\section{Conclusions and Discussions}
In this work we develop and analyze a general bandit algorithm using techniques such as optimistic mirror descent, log-barrier regularizer, increasing learning rate, and so on.
We show various applications of this general framework, obtaining several more adaptive algorithms that improve previous works.
Future directions include 1) improving the dependence on $K$ for the path-length results; 2) obtaining second-order path-length bounds;
3) generalizing the results to the linear bandit problem.

\paragraph{Acknowledgement.}
CYW is grateful for the support of NSF Grant \#1755781. The authors would like to thank Chi-Jen Lu for posing the problem of bandit path-length, and to thank Chi-Jen Lu and Yi-Te Hong for helpful discussions in this direction. 

\bibliography{colt2018-sample} 
\appendix

\section{Proof of Lemma \ref{thm:general_instantaneous}}
\begin{proof}{\textbf{of Lemma \ref{thm:general_instantaneous}}.}
We first state a useful property used in typical OMD analysis. Let $\Omega$ be a convex compact set in $\mathbb{R}^K$, $\psi$ be a convex function on $\Omega$, 
$w'$ be an arbitrary point in $\Omega$, and $x \in \mathbb{R}^K$.
If $w^*=\argmin_{w\in \Omega}\{\inn{w,x}+D_{\psi}(w,w')\}$, then for any $u \in \Omega$,
\begin{align*}
\inn{w^*-u, x}\leq D_{\psi}(u,w')-D_\psi(u,w^*)-D_{\psi}(w^*,w'). 
\end{align*}
This is by the first-order optimality condition of $w^*$ and direct calculations. Applying this to update rule~\eqref{eqn:update_rule_2} we have
\begin{align}
\inn{w_{t+1}^\p-u, \hat{\ell}_t+ a_t} \leq D_{\psi_t}(u,w_{t}^\p)-D_{\psi_t}(u,w_{t+1}^\p)-D_{\psi_t}(w_{t+1}^\p, w_{t}^\p); \label{eqn:apply1}
\end{align}
while applying it to update rule~\eqref{eqn:update_rule_1} and picking $u=w_{t+1}^\p$ we have
\begin{align}
\inn{w_t-w_{t+1}^\p, m_t} \leq D_{\psi_t}(w_{t+1}^\p, w_t^\p)-D_{\psi_t}(w_{t+1}^\p, w_t)-D_{\psi_t}(w_t, w_t^\p).\label{eqn:apply2} 
\end{align}
Now we bound the instantaneous regret as follows:
\begin{align}
&\inn{w_t-u, \hat{\ell}_t}\nonumber \\
&=\inn{w_t-u, \hat{\ell}_t+ a_t}-\inn{w_t, a_t}+\inn{u,  a_t}\nonumber \\
&=\inn{w_t-w_{t+1}^\p, \hat{\ell}_t+a_t}-\inn{w_t, a_t}+\inn{w_{t+1}^\p-u, \hat{\ell}_t+a_t}+\inn{u,   a_t}\nonumber \\
&=\inn{w_t-w_{t+1}^\p, \hat{\ell}_t+a_t-m_t}-\inn{w_t,a_t}+\inn{w_{t+1}^\p-u, \hat{\ell}_t+ a_t}+\inn{w_t-w_{t+1}^\p, m_t}+\inn{u,   a_t} \nonumber \\
&\leq D_{\psi_t}(u,w_{t}^\p)-D_{\psi_t}(u,w_{t+1}^\p)-D_{\psi_t}(w_{t+1}^\p, w_t)-D_{\psi_t}(w_t, w_t^\p)+\inn{u, a_t}, \label{eqn:regret_decomposition}
\end{align}
where last inequality is by the condition $\inn{w_t-w_{t+1}^\p, \hat{\ell}_t+a_t-m_t}-\inn{w_t,a_t}\leq 0$, Eq.~\eqref{eqn:apply1}, and Eq.~\eqref{eqn:apply2}.
\end{proof}

\section{Lemmas for Log-barrier OMD}
\label{section:all_kinds_of_lemmas}

In this section we establish some useful lemmas for update rules~\eqref{eqn:update_rule_1} and~\eqref{eqn:update_rule_2} with log-barrier regularizer,
which are used in the proofs of other theorems.
We start with some definitions.

\begin{definition}
\label{definition:norm}
For any $h \in \mathbb{R}^K$, define norm $\norm{h}_{t,w}=\sqrt{h^\top \nabla^2 \psi_t(w) h}=\sqrt{\sum_{i=1}^K \frac{1}{\eta_{t,i}}\frac{h_i^2}{w_i^2}}$ and its dual norm $\norm{h}_{t,w}^*=\sqrt{h^\top \nabla^{-2} \psi_t(w) h}=\sqrt{\sum_{i=1}^K \eta_{t,i}w_i^2 h_i^2}$.
For some radius $r > 0$, define ellipsoid $\mathcal{E}_{t,w}(r)=\left\{u \in \mathbb{R}^K : \norm{u-w}_{t,w}\leq r \right\}$ . 
\end{definition}

\begin{lemma}
\label{lemma:norm_close}
If $w^\p \in \mathcal{E}_{t,w}(1)$ and $\eta_{t,i}\leq \frac{1}{81}$ for all $i$, then $w_i^\p\in \left[ \frac{1}{2}w_i, \frac{3}{2}w_i \right]$ for all $i$, and also $ 0.9\norm{h}_{t,w} \leq \norm{h}_{t,w^\p} \leq 1.2\norm{h}_{t,w}$ for any $h\in \mathbb{R}^K$. 
\end{lemma}
\begin{proof}
$w^\p\in \mathcal{E}_{t,w}(1)$ implies $\sum_{i=1}^K \frac{1}{\eta_{t,i}}\frac{(w^\p_i-w_i)^2}{w_i^2}\leq 1$. Thus for every $i$, we have $\frac{\abs{w_i^\p-w_i}}{w_i}\leq \sqrt{\eta_{t,i}}\leq \frac{1}{9}$, implying $w_i^\p\in \left[ \frac{8}{9}w_i, \frac{10}{9}w_i \right]\subset\left[ \frac{1}{2}w_i, \frac{3}{2}w_i \right]$. 
Therefore, $\norm{h}_{t,w^\p}
=\sqrt{\sum_{i=1}^K \frac{1}{\eta_{t,i}} \frac{h_i^2}{w^{\p 2}_i}}
\geq \sqrt{\sum_{i=1}^K \frac{1}{\eta_{t,i}}\frac{h_i^2}{\left(\frac{10}{9}w_i\right)^2}}
=0.9\norm{h}_{t,w}$. 
Similarly, we have $\norm{h}_{t,w^\p}\leq 1.2\norm{h}_{t,w}$. 
\end{proof}

\begin{lemma}
\label{lemma:stability}
Let $w_t, w_{t+1}^\p$ follow \eqref{eqn:update_rule_1} and \eqref{eqn:update_rule_2} where $\psi_t$ is the log-barrier with $\eta_{t,i}\leq \frac{1}{81}$ for all $i$. If $\norm{\hat{\ell}_t-m_t+a_t}^*_{t,w_t}\leq \frac{1}{3}$, then $w_{t+1}^\p \in \mathcal{E}_{t,w_t}(1)$. 
\end{lemma}

\begin{proof}
Define $F_{t}(w)=\inn{w, m_t}+D_{\psi_t}(w, w_t^\p)$ and $F_{t+1}^\p(w)=\inn{w, \hat{\ell}_t+a_t}+D_{\psi_t}(w, w_t^\p)$. Then by definition we have $w_t=\argmin_{w\in\Omega}F_{t}(w)$ and $w_{t+1}^\p=\argmin_{w\in\Omega}F_{t+1}^\p(w)$. To show $w_{t+1}^\p\in \mathcal{E}_{t,w_t}(1)$, it suffices to show that for all $u$ on the boundary of $\mathcal{E}_{t,w_t}(1)$, $F^\p_{t+1}(u)\geq F^\p_{t+1}(w_t)$. 

Indeed, using Taylor's theorem, for any $u\in \partial \mathcal{E}_{t,w_t}(1)$, there is an $\xi$ on the line segment between $w_t$ and $u$ such that (let $h\triangleq u-w_t$)
\begin{align*}
F^\p_{t+1}(u)&=F^\p_{t+1}(w_t)+\nabla F^{\p}_{t+1} (w_t)^\top h+ \frac{1}{2}h^\top\nabla^2 F^\p_{t+1}(\xi)h \\
&=F^\p_{t+1}(w_t)+ (\hat{\ell}_t-m_t+a_t)^\top h +\nabla F_t (w_t)^\top h+ \frac{1}{2}h^\top\nabla^2 \psi_t(\xi)h \\
&\geq F^\p_{t+1}(w_t)+ (\hat{\ell}_t-m_t+a_t)^\top h + \frac{1}{2}\norm{h}_{t,\xi}^2 \tag{by the optimality of $w_t$}\\
&\geq F^\p_{t+1}(w_t)+ (\hat{\ell}_t-m_t+a_t)^\top h + \frac{1}{2}\times0.9^2\norm{h}_{t,w_t}^2 \tag{by Lemma \ref{lemma:norm_close}} \\
&\geq F^\p_{t+1}(w_t)- \norm{\hat{\ell}_t-m_t+a_t}^*_{t,w_t} \norm{h}_{t,w_t} + \frac{1}{3}\norm{h}_{t,w_t}^2 \\
&=F^\p_{t+1}(w_t)- \norm{\hat{\ell}_t-m_t+a_t}^*_{t,w_t} + \frac{1}{3} \tag{$\norm{h}_{t,w_t}=1$}\\
&\geq F^\p_{t+1}(w_t). \tag{by the assumption}
\end{align*}
\end{proof}

\begin{lemma}
\label{lemma:stability_under_condition}
Let $w_t, w_{t+1}^\p$ follow \eqref{eqn:update_rule_1} and \eqref{eqn:update_rule_2} where $\psi_t$ is the log-barrier with $\eta_{t,i}\leq \frac{1}{81}$ for all $i$. If $\norm{\hat{\ell}_t-m_t+a_t}^*_{t,w_t}\leq \frac{1}{3}$, then $\norm{w_{t+1}^\p-w_t}_{t,w_t}\leq 3\norm{\hat{\ell}_t-m_t+a_t}_{t,w_t}^*$. 
\end{lemma}
\begin{proof}
Define $F_t(w)$ and $F_{t+1}^\p(w)$ to be the same as in Lemma \ref{lemma:stability}. Then we have 
\begin{align}
F_{t+1}^\p(w_t)-F_{t+1}^\p(w_{t+1}^\p)&=(w_t-w_{t+1}^\p)^\top(\hat{\ell}_t-m_t+a_t) + F_t(w_t)-F_t(w_{t+1}^\p) \nonumber \\
&\leq (w_t-w_{t+1}^\p)^\top(\hat{\ell}_t-m_t+a_t) \nonumber \tag{optimality of $w_t$}\\
&\leq \norm{w_t-w_{t+1}^\p}_{t,w_t}\norm{\hat{\ell}_t-m_t+a_t}_{t,w_t}^*. \label{eqn:direction1}
\end{align}
On the other hand, for some $\xi$ on the line segment between $w_t$ and $w_{t+1}^\p$, we have by Taylor's theorem and the optimality of $w_{t+1}^\p$,
\begin{align}
F_{t+1}^\p(w_t)-F_{t+1}^\p(w_{t+1}^\p)&=\nabla F_{t+1}^\p(w_{t+1}^\p)^\top (w_t-w_{t+1}^\p) + \frac{1}{2}(w_t-w_{t+1}^\p)^\top \nabla^2 F_{t+1}^\p(\xi)(w_t-w_{t+1}^\p) \nonumber \\
&\geq \frac{1}{2}\norm{w_t-w_{t+1}^\p}_{t,\xi}^2 .
\label{eqn:direction2}
\end{align}
Since the condition in Lemma \ref{lemma:stability} holds, $w_{t+1}^\p\in \mathcal{E}_{t,w_t}(1)$, and thus $\xi\in \mathcal{E}_{t,w_t}(1)$. Using again Lemma \ref{lemma:norm_close}, we have 
\begin{align}
\frac{1}{2}\norm{w_t-w_{t+1}^\p}_{t,\xi}^2 \geq \frac{1}{3}\norm{w_t-w_{t+1}^\p}_{t,w_t}^2\label{eqn:direction3}.
\end{align}
Combining \eqref{eqn:direction1}, \eqref{eqn:direction2}, and \eqref{eqn:direction3}, we have $\norm{w_t-w_{t+1}^\p}_{t,w_t}\norm{\hat{\ell}_t-m_t+a_t}_{t,w_t}^* \geq \frac{1}{3}\norm{w_t-w_{t+1}^\p}_{t,w_t}^2$, which leads to the stated inequality. 
\end{proof}

\begin{lemma}
\label{lemma:condition_automatic_hold}
When the three conditions in Theorem \ref{lemma:MAB_condition} hold, we have $\norm{\hat{\ell}_t-m_t+a_t}^{*}_{t,w_t}\leq \frac{1}{3}$ for either $a_{t,i}=6\eta_{t,i}w_{t,i}(\hat{\ell}_{t,i}-m_{t,i})^2$ or $a_{t,i}=0$.  
\end{lemma}
\begin{proof}
For $a_{t,i}=6\eta_{t,i}w_{t,i}(\hat{\ell}_{t,i}-m_{t,i})^2$, we have
\begin{align*}
\norm{\hat{\ell}_t-m_t+a_t}^{*2}_{t,w_t}
&= \sum_{i=1}^K\eta_{t,i}w_{t,i}^2\big(\hat{\ell}_{t,i}-m_{t,i} + 6\eta_{t,i}w_{t,i}(\hat{\ell}_{t,i}-m_{t,i})^2\big)^2 \\
&=\sum_{i=1}^K\eta_{t,i}w_{t,i}^2(\hat{\ell}_{t,i}-m_{t,i})^2+12\eta_{t,i}^2w_{t,i}^3(\hat{\ell}_{t,i}-m_{t,i})^3 +36\eta_{t,i}^3w_{t,i}^4(\hat{\ell}_{t,i}-m_{t,i})^4\\
&\leq \sum_{i=1}^K \eta_{t,i}w_{t,i}^2(\hat{\ell}_{t,i}-m_{t,i})^2(1+36\eta_{t,i}+324\eta_{t,i}^2) \tag{condition (ii)}\\
&\leq 2\sum_{i=1}^K \eta_{t,i}w_{t,i}^2(\hat{\ell}_{t,i}-m_{t,i})^2 \tag{condition (i)}\\
&\leq 2\times \frac{1}{18}=\frac{1}{9}.\tag{condition (iii)}
\end{align*}
For $a_{t,i}=0$, we have
\begin{align*}
\norm{\hat{\ell}_t-m_t+a_t}^{*2}_{t,w_t}=\norm{\hat{\ell}_t-m_t}^{*2}_{t,w_t}=\sum_{i=1}^K \eta_{t,i}w_{t,i}^2(\hat{\ell}_{t,i}-m_{t,i})^2\leq \frac{1}{18} < \frac{1}{9}. \tag{condition (iii)}
\end{align*}
\end{proof}

\begin{lemma}
\label{lemma:2times_bound}
If the three conditions in Theorem \ref{lemma:MAB_condition} hold, \textsc{Broad-OMD} (with either Option I or II)
satisfies $\frac{1}{2}w_{t,i}\leq w^\p_{t+1,i}\leq \frac{3}{2}w_{t,i}$.
\end{lemma}
\begin{proof}
This is a direct application of Lemmas \ref{lemma:condition_automatic_hold},  \ref{lemma:stability}, and \ref{lemma:norm_close}.


\end{proof}

\begin{lemma}
\label{lemma:2times_bound_another}
For the MAB problem, if the three conditions in Theorem \ref{lemma:MAB_condition} hold, \textsc{Broad-OMD} (with either Option I or II)
satisfies $\frac{1}{2}w_{t,i}\leq w^\p_{t,i}\leq \frac{3}{2}w_{t,i}$.
\end{lemma}
\begin{proof}
It suffices to prove $w_{t}^\p \in \mathcal{E}_{t,w_t}(1)$ by Lemma~\ref{lemma:norm_close}.
Since we assume that the three conditions in Theorem \ref{lemma:MAB_condition} hold and $w_t\in \Delta_K$, we have $\norm{m_t}_{t,w_t}^*=\sqrt{\sum_{i=1}^K \eta_{t,i}w_{t,i}^2m_{t,i}^2}\leq \sqrt{\frac{1}{162}\sum_{i=1}^K w_{t,i}^2}\leq \sqrt{\frac{1}{162}}< \frac{1}{3}$. This implies $w_{t}^\p \in \mathcal{E}_{t,w_t}(1)$ by a similar arguments as in the proof of Lemma~\ref{lemma:stability} (one only needs to replace $F_{t+1}^\p(w)$ there by $G(w)\triangleq D_{\psi_t}(w,w_t^\p)$ and note that $w_t^\p=\argmin_{w\in \Delta_K}G(w)$).
\end{proof}


\section{Proof of Theorem \ref{lemma:MAB_condition} and Corollary~\ref{cor:clear_corollary}}

\begin{proof}{\textbf{of Theorem \ref{lemma:MAB_condition}}.}
We first prove Eq.~\eqref{eqn:condition1} holds: by Lemmas \ref{lemma:condition_automatic_hold} 
and \ref{lemma:stability_under_condition}, we have
\begin{align*}
\inn{w_t-w_{t+1}^\p, \hat{\ell}_t-m_t+ a_t}
&\leq \norm{w_t-w_{t+1}^\p}_{t,w_t}\norm{\hat{\ell}_t-m_t+a_t}_{t,w_t}^*\\
&\leq 3\norm{\hat{\ell}_t-m_t+a_t}_{t,w_t}^{*2}\\
&\leq 3\sum_{i=1}^K \eta_{t,i}w_{t,i}^2(\hat{\ell}_{t,i}-m_{t,i})^2(1+36\eta_{t,i}+324\eta_{t,i}^2) \\
&\leq 6\sum_{i=1}^K \eta_{t,i}w_{t,i}^2(\hat{\ell}_{t,i}-m_{t,i})^2 = \inn{w_t, a_t},
\end{align*}
where the last two inequalities are by the same calculations done in the proof of Lemma~\ref{lemma:condition_automatic_hold}.

Since Eq.~\eqref{eqn:condition1} holds, using Lemma~\ref{thm:general_instantaneous} we have (ignoring non-positive terms $-A_t$'s),
\begin{align}
\sum_{t=1}^T\inn{w_t-u, \hat{\ell}_t}&\leq \sum_{t=1}^T\left(D_{\psi_t}(u,w_t^\p)-D_{\psi_t}(u,w^\p_{t+1})\right)+\sum_{t=1}^T\inn{u,a_t}\nonumber \\
&\leq D_{\psi_1}(u, w_1^\p) + \sum_{t=1}^{T}\left( D_{\psi_{t+1}}(u, w^\p_{t+1})-D_{\psi_{t}}(u, w^\p_{t+1}) \right)+\sum_{t=1}^T\inn{u,a_t}.\label{eqn:some_intermediate}
\end{align}
In the last inequality, we add a term $D_{\psi_{T+1}}(u, w_{T+1}^\p) \geq 0$ artificially. As mentioned, $\psi_{T+1}$, defined in terms of $\eta_{T+1,i}$, never appears in the \textsc{Broad-OMD} algorithm. We can simply pick any $\eta_{T+1,i} > 0$ for all $i$ here. This is just to simplify some analysis later. 

The first term in \eqref{eqn:some_intermediate} can be bounded by the optimality of $w_1^\p$:
\begin{align*}
D_{\psi_1}(u, w_1^\p)&=\psi_1(u)-\psi_1(w_1^\p)-\inn{\nabla\psi_1(w_1^\p), u-w_1^\p}\\
&\leq \psi_1(u)-\psi_1(w_1^\p)=\sum_{i=1}^K \frac{1}{\eta_{1,i}}\ln\frac{w_{1,i}^\p}{u_i}.
\end{align*}
The second term, by definition, is
\begin{align*}
\sum_{t=1}^{T}\sum_{i=1}^K \left(\frac{1}{\eta_{t+1,i}}-\frac{1}{\eta_{t,i}}\right) h\left(\frac{u_i}{w_{t+1,i}^\p}\right). 
\end{align*}
Plugging the above two terms into \eqref{eqn:some_intermediate} finishes the proof.
\end{proof}

\begin{proof}{\textbf{of Corollary~\ref{cor:clear_corollary}}.}
We first check the three conditions in Theorem~\ref{lemma:MAB_condition} under our choice of $\eta_{t,i}$ and $\hat{\ell}_{t,i}$: $\eta_{t,i}=\eta=\frac{1}{162K_0}\leq \frac{1}{162}$; $w_{t,i}\abs{\hat{\ell}_{t,i}-m_{t,i}}=\abs{\ell_{t,i}-m_{t,i}}\mathbbm{1}\{i\in b_t\} \leq 2<3$; 
$\sum_{i=1}^K \eta_{t,i}w_{t,i}^2(\hat{\ell}_{t,i}-m_{t,i})^2=\frac{1}{162K_0}\sum_{i=1}^K (\ell_{t,i}-m_{t,i})^2\mathbbm{1}\{i\in b_t\} \leq \frac{4}{162} < \frac{1}{18}$. 
Applying Theorem~\ref{lemma:MAB_condition} we then have 
\begin{align*}
\sum_{t=1}^T \inn{w_t-u, \hat{\ell}_t}\leq \sum_{i=1}^K  \frac{\ln\frac{w^\p_{1,i}}{u_i}}{\eta}  +\sum_{t=1}^T \inn{u,a_t}.
\end{align*}
As mentioned, if we let $u=b^*$, then $\ln \frac{w_{1,i}^\p}{u_i}$
becomes infinity for those $i\notin b^*$. Instead, we let $u=\left(1-\frac{1}{T}\right)b^* + \frac{1}{T}w_1^\p$. With this choice of $u$, we have $\frac{w_{1,i}^\p}{u_i}\leq \frac{w_{1,i}^\p}{\frac{1}{T}w_{1,i}^\p}=T$. Plugging $u$ into the above inequality and rearranging, we get 
\begin{align}
\sum_{t=1}^T \inn{w_t-b^*, \hat{\ell}_t}\leq \frac{K\ln T}{\eta}+\sum_{t=1}^T \inn{b^*,a_t}+B,  \label{eqn:sb_corollary}
\end{align}
where $B\triangleq \frac{1}{T}\sum_{t=1}^T \inn{-b^*+w_1^\p, \hat{\ell}_t+a_t}$. 

Now note that $\mathbb{E}_{b_t}[a_{t,i}]=6\eta (\ell_{t,i}-m_{t,i})^2=\mathcal{O}(\eta)$ and $\mathbb{E}_{b_t}[\hat{\ell}_{t,i}]=\ell_{t,i}=\mathcal{O}(1)$ for all $i$. Thus, $\mathbb{E}[B]=\mathbb{E}\left[\frac{1}{T}\sum_{t=1}^T \inn{-b^*+w_1^\p, \mathbb{E}_{b_t}[\hat{\ell}_t+a_t]}\right] \leq \mathbb{E}\left[\frac{1}{T}\sum_{t=1}^T \norm{-b^*+w_1^\p}_1 \norm{\mathbb{E}_{b_t}[\hat{\ell}_t+a_t]}_\infty\right] = \mathcal{O}(K_0)$. Taking expectation on both sides of \eqref{eqn:sb_corollary}, we have 
\begin{align*}
\mathbb{E}\left[\sum_{t=1}^T b_t^\top \ell_t  - \sum_{t=1}^T b^{*\top} \ell_t \right] \leq \frac{K\ln T}{\eta} + 6\eta\mathbb{E}\left[\sum_{t=1}^T \sum_{i\in b^*}^K (\ell_{t,i}-m_{t,i})^2\right] + \mathcal{O}(K_0). 
\end{align*}
\end{proof}

\section{Proof of Theorem \ref{cor:variance_bound}}
\begin{proof}{\textbf{of Theorem \ref{cor:variance_bound}}.}
As in \cite{hazan2011better}, 
for the rounds we perform uniform sampling we do not update $w_t^\p$. 
Let $\mathcal{S}$ be the set of rounds of uniform sampling. 
Then for the other rounds we can apply Corollary \ref{cor:clear_corollary} to arrive at
\begin{align}
\mathbb{E}\left[\sum_{t\in [T]\backslash \mathcal{S}} \ell_{t,i_t}-\ell_{t,i^*} \right]\leq \frac{K\ln T}{\eta} + 6\eta \mathbb{E}\left[\sum_{t\in [T]\backslash \mathcal{S}}(\ell_{t,i^*}-\tilde{\mu}_{t-1,i^*})^2\right] + \mathcal{O}(1). \label{eqn:regret_bound:a_t_neq_0_another} 
\end{align}
The second term can be bounded as follows: 
\begin{align}
&\mathbb{E}\left[\sum_{t\in [T]\backslash \mathcal{S}} (\ell_{t,i^*}-\tilde{\mu}_{t-1,i^*})^2\right]\leq \mathbb{E}\left[\sum_{t=2}^T (\ell_{t,i^*}-\tilde{\mu}_{t-1,i^*})^2\right] \nonumber \\
&\leq 3\sum_{t=2}^T(\ell_{t,i^*}-\mu_{t,i^*})^2+3\sum_{t=2}^T(\mu_{t,i^*}-\mu_{t-1,i^*})^2 + 3\mathbb{E}\left[\sum_{t=2}^T(\mu_{t-1,i^*}-\tilde{\mu}_{t-1,i^*})^2\right].\label{eqn:decompose_three}
\end{align}
The first and the third terms in \eqref{eqn:decompose_three} can be bounded using Lemma 10 and 11 of \citep{hazan2011better} respectively, and they are both of order $\mathcal{O}(Q_{T,i^*}+1)$ if we pick $M=\Theta(\ln T)$. The second term in \eqref{eqn:decompose_three} can be bounded by a constant by Lemma \ref{lemma:second_Q_term}. Thus second term in \eqref{eqn:regret_bound:a_t_neq_0_another}  can be bounded by $\mathcal{O}\left(\eta (Q_{T,i^*}+1)\right)$. Finally, note that $\mathbb{E}\left[\sum_{t=1}^T \ell_{t,i_t}-\ell_{t,i^*} \right]\leq\mathbb{E}\left[\sum_{t\in [T]\backslash \mathcal{S}} \ell_{t,i_t}-\ell_{t,i^*} \right]+2\mathbb{E}[\abs{\mathcal{S}}]$ and that $\mathbb{E}[\abs{\mathcal{S}}]=\mathcal{O}\left(\sum_{t=1}^T \frac{MK}{t}\right)=\mathcal{O}\left(MK\ln T\right)=\mathcal{O}\left(K(\ln T)^2\right)$. Combining everything, we get 
\begin{align*}
\mathbb{E}\left[\sum_{t=1}^T \ell_{t,i_t}-\ell_{t,i^*} \right]=\mathcal{O}\left( \frac{K\ln T}{\eta} + \eta Q_{T,i^*} + K(\ln T)^2\right).
\end{align*}
\end{proof}

\begin{lemma}
\label{lemma:second_Q_term}
For any $i$, $\sum_{t=2}^T (\mu_{t,i}-\mu_{t-1,i})^2=\mathcal{O}(1)$. 
\end{lemma}
\begin{proof}
By definition, \sloppy$\absolute{\mu_{t,i}-\mu_{t-1,i}}=\absolute{\frac{1}{t}\sum_{s=1}^t \ell_{s,i}-\frac{1}{t-1}\sum_{s=1}^{t-1} \ell_{s,i}}=\absolute{\frac{1}{t}\ell_{t,i}-\frac{1}{t(t-1)}\sum_{s=1}^{t-1}\ell_{s,i}}\leq \absolute{\frac{1}{t}\ell_{t,i}}+\absolute{\frac{1}{t(t-1)}\sum_{s=1}^{t-1}\ell_{s,i}}\leq \frac{2}{t}$. Therefore, $\sum_{t=2}^T (\mu_{t,i}-\mu_{t-1,i})^2\leq \sum_{t=2}^T \frac{4}{t^2}=\mathcal{O}(1)$. 
\end{proof}

\section{Proof of Theorem \ref{thm:path_length}}
We first state a useful lemma.
\begin{lemma}
\label{lemma:bound_ni}
Let $n_i$ be such that $\eta_{T+1,i}=\kappa^{n_i}\eta_{1,i}$, i.e., the number of times the learning rate of arm $i$ changes in \textsc{Broad-OMD+}. Then $n_i\leq \log_2 T$, and $\eta_{t,i}\leq 5\eta_{1,i}$ for all $t,i$.  
\end{lemma}
\begin{proof}
Let $t_1, t_2, \ldots, t_{n_i}\in [T]$ be the rounds the learning rate for arm $i$ changes (i.e., $\eta_{t+1,i}=\kappa \eta_{t,i}$ for $t=t_1, \ldots, t_{n_i}$). 
By the algorithm, we have 
\begin{align*}
KT\geq \frac{1}{\bar{w}_{t_{n_i},i}}>\rho_{t_{n_i},i}>2\rho_{t_{n_i-1},i}>\cdots>2^{n_i-1}\rho_{t_1,i}=2^{n_i}K. 
\end{align*}
Therefore, $n_i\leq \log_2 T$. And we have $\eta_{t,i}\leq \kappa^{\log_2 T}\eta_{1,i}=e^{\frac{\log_2 T}{\ln T}}\eta_{1,i}\leq 5\eta_{1,i}$.
\end{proof}

\begin{proof}{\textbf{of Theorem \ref{thm:path_length}}.}
Again, we verify the three conditions stated in Theorem \ref{lemma:MAB_condition}. By Lemma \ref{lemma:bound_ni}, $\eta_{t,i}\leq 5\eta\leq 5\times\frac{1}{810}=\frac{1}{162}$; also, $w_{t,j}\absolute{\hat{\ell}_{t,j}-m_{t,j}}=w_{t,j}\absolute{\frac{(\ell_{t,j}-m_{t,j})\mathbbm{1}\{i_t=j\}}{\bar{w}_{t,j}}}\leq w_{t,j}\absolute{\frac{2}{w_{t,j}\left(1-\frac{1}{T}\right)}}\leq 3$ because we assume $T\geq 3$; finally, $
\sum_{j=1}^K \eta_{t,j}w_{t,j}^2(\hat{\ell}_{t,j}-m_{t,j})^2=\eta_{t,i_t}w_{t,i_t}^2(\hat{\ell}_{t,i_t}-m_{t,i_t})^2\leq \frac{1}{162}\times 3^2=\frac{1}{18}$.

Let $\tau_j$ denote the last round the learning rate for arm $j$ is updated, that is, $\tau_j\triangleq \max\{t\in [T]: \eta_{t+1,j}=\kappa\eta_{t,j} \}$. 
We assume that the learning rate is updated at least once so that $\tau_j$ is well defined, otherwise one can verify that the bound is trivial.
For any arm $i$ to compete with, let 
$u=\left(1-\frac{1}{T}\right)\mathbf{e}_{i}+\frac{1}{T}w_1^\p
=\left(1-\frac{1}{T}\right)\mathbf{e}_{i}+\frac{1}{KT}\mathbf{1}$, which guarantees $\frac{w_{1,i}^\p}{u_i}\leq T$. Applying Theorem \ref{lemma:MAB_condition}, with $B\triangleq \frac{1}{T}\sum_{t=1}^T \inn{-\mathbf{e}_{i}+w^\p_{1}, \hat{\ell}_t+a_t}$ we have
\begin{align}
\sum_{t=1}^T\inn{w_t, \hat{\ell}_t}-\hat{\ell}_{t,i}&\leq \frac{K\ln T}{\eta} + \sum_{t=1}^{T}\sum_{j=1}^K\left(\frac{1}{\eta_{t+1,j}}-\frac{1}{\eta_{t,j}}\right)h\left(\frac{u_{j}}{w_{t+1,j}^\p}\right)+\sum_{t=1}^T a_{t,i}+B\nonumber \\
&\leq \frac{K\ln T}{\eta} + \left(\frac{1}{\eta_{\tau_i+1,i}}-\frac{1}{\eta_{\tau_i,i}}\right)h\left(\frac{u_{i}}{w_{\tau_i+1,i}^\p}\right)+\sum_{t=1}^T a_{t,i}+B\nonumber \\
&\leq \frac{K\ln T}{\eta} + \frac{1-\kappa}{\eta_{\tau_i+1,i}}h\left(\frac{u_{i}}{w_{\tau_i+1,i}^\p}\right)+\sum_{t=1}^T a_{t,i}+B\nonumber \\
&\leq \frac{K\ln T}{\eta} - \frac{1}{5\eta \ln T}h\left(\frac{u_{i}}{w_{\tau_i+1,i}^\p}\right)+\sum_{t=1}^T a_{t,i}+B,  \label{eqn:quasi_regret_bound1}
\end{align}
where the last inequality is by Lemma~\ref{lemma:bound_ni} and the fact $\kappa-1 \geq \frac{1}{\ln T}$. Now we bound the second and the third term in \eqref{eqn:quasi_regret_bound1} separately. 
\begin{enumerate}
\item For the second term,  by Lemma \ref{lemma:2times_bound} and $T \geq 3$ we have
\begin{align*}
\frac{u_{i}}{w^\p_{\tau_i+1,i}} \geq \frac{1-\frac{1}{T}}{ \frac{3}{2}w_{\tau_i, i} }\geq \frac{\left(1-\frac{1}{T}\right)^2}{\frac{3}{2}\bar{w}_{\tau_i,i}} =\frac{\left(1-\frac{1}{T}\right)^2}{\frac{3}{2}}\times \frac{\rho_{T+1,i}}{2}\geq \frac{\rho_{T+1,i}}{8} \geq \frac{4K}{8} \geq 1.
\end{align*}
Noting that $h(y)$ is an increasing function when $y\geq 1$, we thus have
\begin{align}
h\left(\frac{u_{i}}{w^\p_{\tau_i+1,i}}\right)\geq h\left(\frac{\rho_{T+1,i}}{8}\right)
=\frac{\rho_{T+1,i}}{8}-1-\ln\left(\frac{\rho_{T+1,i}}{8}\right)\geq \frac{\rho_{T+1,i}}{8}-1-\ln\left(\frac{KT}{4}\right). \label{eqn:path_length_second_term}
\end{align}

\item For the third term, we proceed as
\begin{align}
\sum_{t=1}^T a_{t,i} &= 6\sum_{t=1}^T \eta_{t,i}w_{t,i}(\hat{\ell}_{t,i}-m_{t,i})^2\leq 90\eta \sum_{t=1}^T \abs{\hat{\ell}_{t,i}-m_{t,i}}  \nonumber \\
&\leq 90\eta\left(\max_{t\in[T]}\frac{1}{\bar{w}_{t,i}}\right) \sum_{t=1}^{T}  \abs{\ell_{t,i}-\ell_{t-1,i}} \leq 90\eta\rho_{T+1,i} V_{T,i}, \label{eqn:path_length_third_term}
\end{align}
where in the first inequality, we use $w_{t,i}\abs{\hat{\ell}_{t,i}-m_{t,i}}\leq 3$ and $\eta_{t,i}\leq 5\eta$; in the second inequality, we do a similar calculation as in Eq.~\eqref{eqn:path_length_trick} (only replacing $w_{t,i}$ by $\bar{w}_{t,i}$); and in the last inequality, we use the fact $\frac{1}{\bar{w}_{t,i}}\leq \rho_{T+1,i}$ for all $t\in [T]$
by the algorithm.
\end{enumerate}
Combining Eq.~\eqref{eqn:path_length_second_term} and Eq.~\eqref{eqn:path_length_third_term} and using the fact $\frac{1+\ln\left(\frac{KT}{4}\right)}{5\ln T}\leq K\ln T$, we continue from Eq.~\eqref{eqn:quasi_regret_bound1} to arrive at
\begin{align}
\sum_{t=1}^T \inn{w_t, \hat{\ell}_t}-\hat{\ell}_{t,i}\leq \frac{2K\ln T}{\eta}+ \rho_{T+1,i}\left( \frac{-1}{40\eta\ln T} +90\eta V_{T,i} \right) +B,  \label{eqn:quasi_regret_bound2}
\end{align}
We are almost done here, but note that the left-hand side of \eqref{eqn:quasi_regret_bound2} is not the desired regret. What we would like to bound is
\begin{align}
\sum_{t=1}^T \inn{\bar{w}_t, \hat{\ell}_t} - \sum_{t=1}^T \hat{\ell}_{t,i}=\sum_{t=1}^T \inn{\bar{w}_t-w_t, \hat{\ell}_t}+ \sum_{t=1}^T\left(\inn{w_t, \hat{\ell}_t}-\hat{\ell}_{t,i}\right), \label{eqn:quasi_regret_bound3}
\end{align}
where the second summation on the right-hand side is bounded by Eq.~\eqref{eqn:quasi_regret_bound2}.
The first term can be written as $\sum_{t=1}^T \inn{-\frac{1}{T}w_t+\frac{1}{KT}\mathbf{1}, \hat{\ell}_t}$. Note that$
\frac{1}{T}\sum_{t=1}^T\inn{-w_t, \hat{\ell}_t}\leq \frac{1}{T}\sum_{t=1}^T\abs{\inn{w_t,\hat{\ell}_t-m_t}}+\frac{1}{T}\sum_{t=1}^T\abs{\inn{w_t, m_t}} \leq 3 + 1=4$, and
$
\mathbb{E}\left[\frac{1}{T}\sum_{t=1}^T \inn{\frac{1}{K}\mathbf{1},\hat{\ell}_{t}}\right]=\frac{1}{T}\sum_{t=1}^T \inn{\frac{1}{K}\mathbf{1},\ell_t}\leq 1. $
Therefore, taking expectation on both sides of \eqref{eqn:quasi_regret_bound3}, we get 
\begin{align*}
\mathbb{E}\left[\sum_{t=1}^T \ell_{t,i_t} \right] - \sum_{t=1}^T \ell_{t,i} \leq \frac{2K\ln T}{\eta} + \mathbb{E}[\rho_{T+1,i}]\left( \frac{-1}{40\eta\ln T} +90\eta V_{T,i} \right) + \mathcal{O}(1),   
\end{align*}
because $\mathbb{E}[B]$ is also $\mathcal{O}(1)$ as proved in Corollary~\ref{cor:clear_corollary}. 
\end{proof}

\section{Proofs of Lemma \ref{lemma:simple_lemma} and Theorem \ref{lemma:second_order_regret_bound}}

\begin{proof}{\textbf{of Lemma \ref{lemma:simple_lemma}}.}
By the same arguments as in the proof of Lemma~\ref{thm:general_instantaneous}, we have
\begin{align*}
\inn{w_{t+1}^\p-u, \hat{\ell}_t} \leq D_{\psi_t}(u,w_{t}^\p)-D_{\psi_t}(u,w_{t+1}^\p)-D_{\psi_t}(w_{t+1}^\p, w_{t}^\p); 
\end{align*}
and 
\begin{align*}
\inn{w_t-w_{t+1}^\p, m_t} \leq D_{\psi_t}(w_{t+1}^\p, w_t^\p)-D_{\psi_t}(w_{t+1}^\p, w_t)-D_{\psi_t}(w_t, w_t^\p).
\end{align*}
Therefore, by expanding the instantaneous regret, we have
\begin{align*}
&\inn{w_t-u, \hat{\ell}_t}\nonumber \\
&=\inn{w_t-w_{t+1}^\p, \hat{\ell}_t-m_t}+\inn{w_{t+1}^\p-u, \hat{\ell}_t}+\inn{w_t-w_{t+1}^\p, m_t} \nonumber \\
&\leq \inn{w_t-w_{t+1}^\p, \hat{\ell}_t-m_t} + D_{\psi_t}(u,w_{t}^\p)-D_{\psi_t}(u,w_{t+1}^\p)-D_{\psi_t}(w_{t+1}^\p, w_t)-D_{\psi_t}(w_t, w_t^\p). 
\end{align*}
\end{proof}
\begin{proof}{\textbf{of Theorem \ref{lemma:second_order_regret_bound}}.}
Applying Lemma \ref{lemma:simple_lemma}, we have 
\begin{align*}
\sum_{t=1}^T\inn{w_t-u, \hat{\ell}_t} &\leq \sum_{t=1}^T \left(D_{\psi_t}(u,w_{t}^\p)-D_{\psi_t}(u,w_{t+1}^\p)+\inn{w_t-w_{t+1}^\p, \hat{\ell}_t-m_t}-A_t\right) \\
&\leq \sum_{i=1}^K \frac{\ln \frac{w_{1,i}^\p}{u_i}}{\eta} +\sum_{t=1}^T \inn{w_t-w_{t+1}^\p, \hat{\ell}_t-m_t}-A_t .
\end{align*}
For the second term, using Lemma \ref{lemma:condition_automatic_hold} and \ref{lemma:stability_under_condition} we bound $\inn{w_t-w_{t+1}^\p, \hat{\ell}_t-m_t}$ by
\begin{align*}
\norm{w_t-w_{t+1}^\p}_{t,w_t}\norm{\hat{\ell}_t-m_t}_{t,w_t}^* 
\leq 3\norm{\hat{\ell}_t-m_t}_{t,w_t}^{*2} = 3\eta\sum_{i=1}^K w_{t,i}^2(\hat{\ell}_{t,i}-m_{t,i})^2
\end{align*}

Finally we lower bound $A_t$ for the MAB case. Note $h(y)=y-1-\ln y\geq \frac{(y-1)^2}{6}$ for $y\in [\frac{1}{2},2]$. By Lemma~\ref{lemma:2times_bound} and \ref{lemma:2times_bound_another}, $\frac{w_{t+1,i}^\p}{w_{t,i}}$ and $\frac{w_{t,i}}{w_{t,i}^\p}$ both belong to $[\frac{1}{2},2]$. Therefore, 
\begin{align*}
A_t&=D_{\psi_t}(w_{t+1}^\p, w_t)+D_{\psi_t}(w_t, w_t^\p)=\frac{1}{\eta} \sum_{i=1}^K \left(h\left(\frac{w_{t+1,i}^\p}{w_{t,i}}\right) +h\left(\frac{w_{t,i}}{w_{t,i}^\p}\right)\right) \\
&\geq \frac{1}{6\eta} \sum_{i=1}^K \left( \frac{(w_{t+1,i}^\p-w_{t,i})^2}{w_{t,i}^2} + \frac{(w_{t,i}-w_{t,i}^\p)^2}{w_{t,i}^{\p 2}} \right) \\
&\geq \frac{1}{24\eta} \sum_{i=1}^K \left( \frac{(w_{t+1,i}^\p-w_{t,i})^2}{w_{t,i}^2} + \frac{(w_{t,i}-w_{t,i}^\p)^2}{w_{t-1,i}^2} \right), 
\end{align*}
and 
\begin{align*}
\sum_{t=1}^T A_t &\geq \frac{1}{24\eta}\sum_{t=2}^{T}\sum_{i=1}^K\frac{(w_{t,i}^\p-w_{t-1,i})^2}{w_{t-1,i}^2}+\sum_{t=2}^T \sum_{i=1}^K \frac{(w_{t,i}-w_{t,i}^\p)^2}{w_{t-1,i}^2}\geq \frac{1}{48\eta}\sum_{t=2}^T \sum_{i=1}^K \frac{(w_{t,i}-w_{t-1,i})^2}{w_{t-1,i}^2}. 
\end{align*}
\end{proof}

\section{Doubling Trick}
\label{app:doubling_trick}

\begin{algorithm}[t]
\DontPrintSemicolon
\caption{Doubling trick for \textsc{Broad-OMD} with $a_t=\mathbf{0}$}
\label{alg:doubling}
\textbf{Initialize}: $\eta=\frac{1}{162K_0}, T_0=0, t=1.$\\
\For{$\beta=0, 1, \ldots$}{
   $w_{t}^\p=\argmin_{w\in \Omega}\psi_1(w)$ (restart \textsc{Broad-OMD}). \\
   \While{$t\leq T$}{
      Update $w_t$, sample $b_t\sim w_t$, and update $w_{t+1}^\p$ as in \textsc{Broad-OMD} with Option II. \\
      \If{$\sum_{s=T_\beta+1}^{t} \sum_{i=1}^K w_{s,i}^2(\hat{\ell}_{s,i}-m_{s,i})^2 \geq \frac{K\ln T}{3\eta^2}$}{
          $\eta \leftarrow \eta/2$, $T_{\beta+1} \leftarrow t$, $t\leftarrow t+1$. \\
          \textbf{break}.
      }
      $t\leftarrow t+1$. 
   }
}
\end{algorithm}

We include the version of our algorithm with the doubling trick in Algorithm~\ref{alg:doubling}.
For simplicity we still assume the time horizon $T$ is known; the extension to unknown horizon is straightforward. 

\begin{proof}{\textbf{of Theorem \ref{thm:doubling_trick_theorem}}.}
Let $u=\left(1-\frac{1}{T}\right)b^*+\frac{1}{T}w_1^\p$ so that $\ln \frac{w^\p_{1,i}}{u_i} \leq \ln T$.
At some epoch $\beta$, by Theorem~\ref{lemma:second_order_regret_bound}, the break condition, and condition (iii) we have with $\eta_\beta\triangleq\frac{2^{-\beta}}{162K_0}$,
\begin{align*}
\sum_{t= T_\beta+1}^{T_{\beta+1}}\inn{w_t-u, \hat{\ell}_t} &\leq \frac{K\ln T}{\eta_\beta} + 
3\eta_\beta\sum_{t=T_\beta+1}^{T_{\beta+1}}\sum_{i=1}^K w_{t,i}^2(\hat{\ell}_{t,i}-m_{t,i})^2 \\
& \leq  \frac{2K\ln T}{\eta_\beta} + 3\eta_\beta \sum_{i=1}^K w_{T_{\beta+1},i}^2(\hat{\ell}_{T_{\beta+1},i}-m_{T_{\beta+1},i})^2 
= \mathcal{O}\left(\frac{K\ln T}{\eta_\beta}\right).
\end{align*}

Suppose that at time $T$, the algorithm is at epoch $\beta=\beta^*$. Then we have
\[
\sum_{t=1}^T\inn{w_t-u, \hat{\ell}_t}\leq\sum_{\beta=0}^{\beta^*}\mathcal{O}\left( \frac{K\ln T}{\eta_\beta} \right)\leq \sum_{\beta=0}^{\beta^*}\mathcal{O}\left( 2^\beta K_0 K\ln T \right)\leq\mathcal{O}\left(2^{\beta^*}K_0 K\ln T\right).
\]
It remains to bound $\beta^*$.
If $\beta^*=0$ (no restart ever happened), then trivially $\sum_{t= 1}^{T}\inn{w_t-u, \hat{\ell}_t}=\mathcal{O}(K_0 K\ln T)$. 
Otherwise, because epoch $\beta^*-1$ finishes, we have 
\begin{align*}
\sum_{t= T_{\beta^*-1}+1}^{T_{\beta^*}}\sum_{i=1}^K w_{t,i}^2(\hat{\ell}_{t,i}-m_{t,i})^2 \geq \frac{K\ln T}{3(\eta_{\beta^*-1})^2} = \Omega(2^{2\beta^*}K_0^2K\ln T).  
\end{align*}
Combining them, we have 
\begin{align}
\sum_{t=1}^T\inn{w_t-u, \hat{\ell}_t} &\leq \mathcal{O}\left(2^{\beta^*}K_0K\ln T\right)
\leq \mathcal{O}\left( \sqrt{(K\ln T)\sum_{t= T_{\beta^*-1}+1}^{T_{\beta^*}}\sum_{i=1}^K w_{t,i}^2(\hat{\ell}_{t,i}-m_{t,i})^2} \right) \nonumber \\
&\leq \mathcal{O}\left( \sqrt{(K\ln T)\sum_{t= 1}^{T}\sum_{i=1}^K w_{t,i}^2(\hat{\ell}_{t,i}-m_{t,i})^2} \right), \label{eqn:doubling_bound_1}
\end{align}
Combining both cases we have 
\begin{align}
\sum_{t=1}^T\inn{w_t-u, \hat{\ell}_t} \leq \mathcal{O}\left( \sqrt{K\ln T\sum_{t= 1}^{T}\sum_{i=1}^K w_{t,i}^2(\hat{\ell}_{t,i}-m_{t,i})^2} + K_0K\ln T\right).
\end{align}
Now substituting $u$ by its definition and taking expectations, with $B\triangleq \frac{1}{T} \sum_{t=1}^T \inn{-b^*+w_1^\p, \hat{\ell}_{t}}$ we arrive at 
\begin{align*}
\mathbb{E}\left[ \sum_{t=1}^T \inn{b_t-b^*, \ell_t} \right]
&\leq \mathcal{O}\left(\mathbb{E}\left[\sqrt{K\ln T\sum_{t= 1}^{T}\sum_{i=1}^K w_{t,i}^2(\hat{\ell}_{t,i}-m_{t,i})^2}\right]+K_0K\ln T \right)+\mathbb{E}[B] \\
&\leq \mathcal{O}\left(  \sqrt{K\ln T\mathbb{E}\left[\sum_{t= 1}^{T}\sum_{i=1}^K w_{t,i}^2(\hat{\ell}_{t,i}-m_{t,i})^2\right]}+K_0K\ln T \right), 
\end{align*}
where the last inequality uses the fact $\mathbb{E}[B]=\mathcal{O}(K)$ and Jensen's inequality.
\end{proof}

\section{Proofs of Corollary \ref{cor:path_length_bound_1} and Theorem \ref{thm:fast_convergence_theorem}}

\begin{proof}{\textbf{of Corollary~\ref{cor:path_length_bound_1}}.}
We first verify the three conditions in Theorem~\ref{lemma:second_order_regret_bound}: $\eta \leq \frac{1}{162}$ by assumption; $w_{t,i}\absolute{\hat{\ell}_{t,i}-m_{t,i}}=\absolute{(\ell_{t,i}-\ell_{\alpha_i(t),i})\mathbbm{1}\{i_t=i\}}\leq 2<3$; $\eta \sum_{i=1}^K w_{t,i}^2(\hat{\ell}_{t,i}-m_{t,i})^2=\eta w_{t,i_t}^2(\hat{\ell}_{t,i_t}-m_{t,i_t})^2\leq \frac{9}{162}=\frac{1}{18}$. Let $u=\left(1-\frac{1}{T}\right)\mathbf{e}_{i^*} + \frac{1}{T}w_1^\p$, which guarantees $\frac{w_{1,i}^\p}{u_i}\leq T$. By Theorem~\ref{lemma:second_order_regret_bound} and some rearrangement, we have
\begin{align*}
\sum_{t=1}^T \inn{w_t-\mathbf{e}_{i^*}, \hat{\ell}_t}\leq \frac{K\ln T}{\eta}  +3\eta\sum_{t=1}^T\sum_{i=1}^K  w_{t,i}^2(\hat{\ell}_{t,i}-m_{t,i})^2-\sum_{t=1}^T A_t+B, 
\end{align*}
where $B\triangleq \frac{1}{T}\sum_{t=1}^T \inn{-\mathbf{e}_{i^*}+w_1^\p, \hat{\ell}_t}$. To get the stated bound, just note that $\mathbb{E}[B]=\mathcal{O}(1)$,  and replace $\sum_{t=1}^T\sum_{i=1}^K w_{t,i}^2(\hat{\ell}_{t,i}-m_{t,i})^2$ by the upper bound at \eqref{eqn:path_length_calculation_1} and $A_t$ by the lower bound in Theorem~\ref{lemma:second_order_regret_bound}.
\end{proof}

\section{Omitted Details in Section~\ref{subsection:games}}
\label{appendix:game}
Although the generalization to multi-player games is straightforward, for simplicity we only consider two-player zero-sum games.

We first describe the protocol of the game. The game is defined by an unknown matrix $G\in[-1,1]^{M\times N}$
where entry $G(i,j)$ specifies the loss (or reward) for Player 1 (or Player 2) if Player 1 picks row $i$ while Player 2 picks column $j$.
The players play the game repeatedly for $T$ rounds.
At round $t$, Player 1 randomly picks a row $i_t \sim x_t$ for some $x_t \in \Delta_M$
while Player 2 randomly picks a column $j_t \sim y_t$ for some $y_t \in \Delta_N$.
In~\citep{syrgkanis2015fast}, the feedbacks they receive are the vectors $Gy_t$ and $x_t^\top G$ respectively.
As a natural extension to the bandit setting, we consider a setting where the feedbacks are the scalar values $\mathbf{e}_{i_t}^\top Gy_t$
and $x_t^\top G\mathbf{e}_{j_t}$ respectively, that is, the expected loss/reward for the players' own realized actions (over the opponent's randomness). 

It is clear that each player is essentially facing an MAB problem and thus can employ an MAB algorithm.
Specifically, if both players apply Exp3 for example, their expected average strategies converge to a Nash equilibrium at rate $1/\sqrt{T}$.
However, if instead Player 1 applies \textsc{Broad-OMD} configured as in Corollary~\ref{cor:path_length_bound_1},
then her regret has a path-length term that can be bounded as follows:
\begin{align*}
\sum_{i=1}^K \sum_{t=2}^T\left| \mathbf{e}_{i}^\top Gy_t -  \mathbf{e}_{i}^\top Gy_{t-1}\right|
\leq \sum_{i=1}^K \sum_{t=2}^T\left\| \mathbf{e}_{i}^\top G \right\|_\infty \|y_t - y_{t-1}\|_1 \leq K \sum_{t=2}^T \|y_t - y_{t-1}\|_1,
\end{align*}
which is closely related to the negative regret term in Corollary~\ref{cor:path_length_bound_1}
for Player 2 if she also employs the same \textsc{Broad-OMD}.
The cancellation of these terms then lead to faster convergence rate.

\begin{theorem}
\label{thm:fast_convergence_theorem}
For the setting described above, if both players run \textsc{Broad-OMD} configured as in Corollary~\ref{cor:path_length_bound_1} except that $\eta_{t,i}=\eta= (M+N)^{-\frac{1}{4}}T^{-\frac{1}{4}}$, then their expected average strategies converge to Nash equilibriums at the rate of $\tilde{\mathcal{O}}\left((M+N)^{\frac{5}{4}}/T^{\frac{3}{4}}\right)$, that is,
\begin{align*}
\max_{y\in \Delta_N} \mathbb{E}[\bar{x}]^\top Gy \leq \text{\rm Val} + \tilde{\mathcal{O}}((M+N)^{\frac{5}{4}}/T^{\frac{3}{4}}) \quad\text{and}\quad
\min_{x\in \Delta_M}x^\top G\mathbb{E}[\bar{y}] \geq \text{\rm Val} - \tilde{\mathcal{O}}((M+N)^{\frac{5}{4}}/T^{\frac{3}{4}}),
\end{align*}
where $\bar{x}=\frac{1}{T}\sum_{t=1}^T x_t, \bar{y}=\frac{1}{T}\sum_{t=1}^T y_t$ and 
$\text{\rm Val}= \min\limits_{x\in \Delta_M}\max\limits_{y\in \Delta_N} x^\top Gy = \max\limits_{y\in \Delta_N}\min\limits_{x\in \Delta_M} x^\top Gy$.
\end{theorem}

\begin{proof}
As mentioned, Player 1's $V_{T,i}$ is 
\begin{align*}
\sum_{t=1}^T \abs{\ell_{t,i}-\ell_{t-1,i}}=&\sum_{t=1}^T \abs{\mathbf{e}_i^\top Gy_t-\mathbf{e}_i^\top Gy_{t-1}}\leq \sum_{t=1}^T \norm{\mathbf{e}_i^\top G}_\infty \norm{y_t-y_{t-1}}_1\leq \sum_{t=1}^T \norm{y_t-y_{t-1}}_1
\end{align*}
due to the assumption $|G(i,j)|\leq 1$. Therefore, by Corollary \ref{cor:path_length_bound_1}, Player 1's (pseudo) regret is
\begin{align*}
&\max_{x \in \Delta_M} \mathbbm{E}\left[\sum_{t=1}^T x_t^\top G y_t - \sum_{t=1}^T x^\top G y_t \right] \\
&\leq \mathcal{O}\left(\frac{M\ln T}{\eta}\right) + \mathbb{E}\left[6\eta M \sum_{t=1}^T \norm{y_t-y_{t-1}}_1 -\frac{1}{48\eta}\sum_{t=2}^T \sum_{i=1}^M \frac{(x_{t,i}-x_{t-1,i})^2}{x_{t-1,i}^2} \right],
\end{align*}
while Player 2's (pseudo) regret is
\begin{align*}
&\max_{y \in \Delta_N} \mathbbm{E}\left[\sum_{t=1}^T x_T^\top G y - \sum_{t=1}^T x_t^\top G y_t \right] \\
&\leq \mathcal{O}\left(\frac{N\ln T}{\eta}\right) + \mathbb{E}\left[6\eta N \sum_{t=1}^T \norm{x_t-x_{t-1}}_1 -\frac{1}{48\eta}\sum_{t=2}^T \sum_{i=1}^N \frac{(y_{t,i}-y_{t-1,i})^2}{y_{t-1,i}^2} \right].
\end{align*}
Summing up the above two bounds, and using the following fact (by the inequality $a - b \leq \frac{a^2}{4b}$):
\begin{align*}
\sum_{i=1}^N \left(6\eta M \abs{y_{t,i}-y_{t-1,i}}- \frac{(y_{t,i}-y_{t-1,i})^2}{48\eta y_{t-1,i}^2}\right)\leq 432\eta^3 M^2 \sum_{i=1}^N y_{t-1,i}^2 \leq 432\eta^3 M^2, 
\end{align*}
we get 
\begin{align*}
\max_{y \in \Delta_N}\mathbb{E}[\bar{x}]^\top Gy - \min_{x \in \Delta_M} x^\top G \mathbb{E}[\bar{y}] 
 =\mathcal{O}\left(\frac{(M+N)\ln T}{T\eta} + \eta^3 (M^2+N^2) \right).
\end{align*}
With $\eta=\tilde{\Theta}\left( (M+N)^{-\frac{1}{4}}T^{-\frac{1}{4}}\right)$ the above bound becomes $\tilde{\mathcal{O}}\left((M+N)^{\frac{5}{4}}T^{-\frac{3}{4}}\right)$.
Rearranging then gives
\begin{align*}
\max_{y\in\Delta_N} \mathbb{E}[\bar{x}]^\top Gy 
&\leq \min_{x \in \Delta_M} x^\top G \mathbb{E}[\bar{y}]  + \tilde{\mathcal{O}}((M+N)^{\frac{5}{4}}T^{-\frac{3}{4}}), \\
&\leq \min_{x\in \Delta_M}\max_{y\in \Delta_N} x^\top Gy + \tilde{\mathcal{O}}((M+N)^{\frac{5}{4}}T^{-\frac{3}{4}})
= \text{\rm Val} + \tilde{\mathcal{O}}((M+N)^{\frac{5}{4}}T^{-\frac{3}{4}}), 
\end{align*}
and similarly
\begin{align*}
\min_{x\in \Delta_M} x^\top G\mathbb{E}[\bar{y}]
&\geq \max_{y\in\Delta_N} \mathbb{E}[\bar{x}]^\top Gy - \tilde{\mathcal{O}}((M+N)^{\frac{5}{4}}T^{-\frac{3}{4}}) \\
&\geq \max_{y\in \Delta_N}\min_{x\in \Delta_M}x^\top Gy - \tilde{\mathcal{O}}((M+N)^{\frac{5}{4}}T^{-\frac{3}{4}})
= \text{\rm Val} - \tilde{\mathcal{O}}((M+N)^{\frac{5}{4}}T^{-\frac{3}{4}}), 
\end{align*}
completing the proof.
\end{proof}


As shown by the theorem, we obtain convergence rate faster than $1/\sqrt{T}$,
but still slower than the $1/T$ rate compared to the full-information setup of~\citep{rakhlin2013optimization, syrgkanis2015fast},
due to the fact that we only have first-order instead of second-order path-length bound.

Note that~\citet{rakhlin2013optimization} also studies two-player zero-sum games with bandit feedback
but with an unnatural restriction that in each round the players play the same strategy for four times.
\citet{foster2016learning} greatly weakened the restriction, but their algorithm only converges to some approximation of Val.
For further comparisons, the readers are referred to the comparisons to~\citep{syrgkanis2015fast}
in \citep{foster2016learning}.
We also point out that the question raised in \citep{rakhlin2013optimization} remains open: if the players only receive the realized loss/reward $\mathbf{e}_{i_t}^\top G\mathbf{e}_{j_t}$ as feedback (a more natural setup), can the convergence rate to Val be faster than $1/\sqrt{T}$?

\section{Proof of Theorem \ref{thm:best of both}}

\begin{proof}{\textbf{of Theorem \ref{thm:best of both}}.}
We first verify conditions (ii) and (iii) in Theorem~\ref{thm:doubling_trick_theorem} hold for $\hat{\ell}_{t,i}=\frac{\ell_{t,i}\mathbbm{1}\{i_t=i\}}{w_{t,i}}$ and $m_{t,i}=\ell_{t,i_t}$. Indeed, condition (ii) holds since $w_{t,i}\abs{\hat{\ell}_{t,i}-m_{t,i}}=\abs{\ell_{t,i}\mathbbm{1}\{i_t=i\}-w_{t,i}\ell_{t,i_t}}\leq 2<3$.
Other the other hand, condition (iii) also holds because
\begin{align*}
\eta\sum_{i=1}^K w_{t,i}^2(\hat{\ell}_{t,i}-m_{t,i})^2
&=\eta\sum_{i=1}^K  (\ell_{t,i}\mathbbm{1}\{i_t=i\}-w_{t,i}\ell_{t,i_t})^2\\
&=\eta\sum_{i=1}^K (\ell_{t,i}^2\mathbbm{1}\{i_t=i\}-2\ell_{t,i}w_{t,i}\ell_{t,i_t}\mathbbm{1}\{i_t=i\}+w_{t,i}^2\ell_{t,i_t}^2)\\
&\leq \frac{1}{162}\left(\ell_{t,i_t}^2-2w_{t,i_t}\ell_{t,i_t}^2 + \left(\sum_{i=1}^K w_{t,i}^2\right)\ell_{t,i_t}^2\right)\\
&\leq \frac{1}{162}\left(1+0+1\right) < \frac{1}{18}. 
\end{align*}
Thus, 
by Theorem~\ref{thm:doubling_trick_theorem}, we have
\begin{align}
\mathbb{E}\left[\sum_{t=1}^T \ell_{t,i_t}-\ell_{t,i^*}\right]=\mathcal{O}\left( \sqrt{(K\ln T) \mathbb{E}\left[\sum_{t=1}^T \sum_{i=1}^K w_{t,i}^2(\hat{\ell}_{t,i}-\ell_{t,i_t})^2\right]} + K\ln T \right). \label{eqn:squint_like_further}
\end{align}

Now we consider the stochastic setting. In this case, we further take expectations over $\ell_1, \ldots, \ell_T$ on both sides of \eqref{eqn:squint_like_further}. The left-hand side of \eqref{eqn:squint_like_further} can be lower bounded by
\begin{align}
\mathbb{E}\left[\sum_{t=1}^T \ell_{t,i_t}-\ell_{t,i^*}\right]
&=\mathbb{E}\left[\sum_{t=1}^T \ell_{t,i_t}-\min_{j}\sum_{t=1}^T\ell_{t,j}\right]
\geq \mathbb{E}\left[\sum_{t=1}^T \ell_{t,i_t}-\sum_{t=1}^T\ell_{t,a^*}\right] \nonumber\\
&=\mathbb{E}\left[\sum_{t=1}^T \sum_{i=1}^K w_{t,i}(\ell_{t,i}-\ell_{t,a^*})\right]
\geq \mathbb{E}\left[\sum_{t=1}^T \sum_{i\neq a^*} w_{t,i}\Delta\right]=\Delta\mathbb{E}\left[ \sum_{t=1}^T (1-w_{t,a^*}) \right]. \label{eqn:best_one_direction}
\end{align}
On the other hand, 
\begin{align}
&\mathbb{E}_{i_t\sim w_t}\left[\sum_{i=1}^Kw_{t,i}^2(\hat{\ell}_{t,i}-\ell_{t,i_t})^2\right]
=\mathbb{E}_{i_t\sim w_t}\left[\sum_{i=1}^Kw_{t,i}^2\left(\frac{\ell_{t,i}\mathbbm{1}\{i_t=i\}}{w_{t,i}}-\ell_{t,i_t}\right)^2\right] \nonumber \\
&=\mathbb{E}_{i_t\sim w_t}\left[\sum_{i=1}^K\left(\ell_{t,i}\mathbbm{1}\{i_t=i\}-w_{t,i}\ell_{t,i_t}\right)^2\right] \nonumber \\
&=\sum_{i=1}^K \left( w_{t,i}\left( \ell_{t,i}-w_{t,i}\ell_{t,i} \right)^2 +\sum_{j\neq i} w_{t,j}(w_{t,i}\ell_{t,j})^2 \right) \nonumber \\
&\leq \sum_{i=1}^K \left( w_{t,i}\left( 1-w_{t,i}\right)^2 +\sum_{j\neq i} w_{t,j}w_{t,i}^2 \right)=\sum_{i=1}^K w_{t,i}(1-w_{t,i}) \nonumber \\
&\leq (1-w_{t,a^*}) + \sum_{i\neq a^*} w_{t,i} = 2(1-w_{t,a^*}). \label{eqn:regret_expansion_best_both}
\end{align}
Therefore, the first term on the right-hand side of \eqref{eqn:squint_like_further} can be upper bounded by
\begin{align}
\sqrt{(K\ln T)\mathbb{E}\left[\sum_{t=1}^T \sum_{i=1}^K w_{t,i}^2(\hat{\ell}_{t,i}-\ell_{t,i_t})^2\right]}\leq \sqrt{(K\ln T) \mathbb{E}\left[\sum_{t=1}^T 2(1-w_{t,a^*}) \right]}. \label{eqn:best_another_direction}
\end{align}
Let $H=\mathbb{E}\left[\sum_{t=1}^T (1-w_{t,a^*}) \right]$. Combining \eqref{eqn:best_one_direction}, \eqref{eqn:best_another_direction}, and \eqref{eqn:squint_like_further}, we have 
\begin{align*}
H\Delta \leq \mathcal{O}\left(\sqrt{(K\ln T)H} + K\ln T\right), 
\end{align*}
which implies $H=\mathcal{O}\left(\frac{K\ln T}{\Delta^2}\right)$. Therefore, the expected regret is upper bounded by 
\[\mathcal{O}\left( \sqrt{(K\ln T)H}+K\ln T \right) = \mathcal{O}\left(\frac{K\ln T}{\Delta}\right) .\]

For the adversarial setting, we continue from an intermediate step of \eqref{eqn:regret_expansion_best_both}: 
\begin{align*}
&\mathbb{E}_{i_t\sim w_t}\left[\sum_{i=1}^Kw_{t,i}^2(\hat{\ell}_{t,i}-\ell_{t,i_t})^2\right]
=\sum_{i=1}^K \left( w_{t,i}(1-w_{t,i})^2\ell_{t,i}^2 +\sum_{j\neq i} w_{t,j}w_{t,i}^2\ell_{t,j}^2 \right) \\
&\leq \sum_{i=1}^K w_{t,i}\ell_{t,i}^2 + \sum_{j=1}^K \sum_{i\neq j} w_{t,j}w_{t,i}^2\ell_{t,j}^2\leq \sum_{i=1}^K w_{t,i}\ell_{t,i}^2 + \sum_{j=1}^K w_{t,j}\ell_{t,j}^2 = 2\mathbb{E}_{i_t\sim w_t}\left[\ell_{t,i_t}^2\right]
\end{align*}
Assuming $\ell_{t,i}\in [0,1]$, we thus have $\ell_{t,i_t}^2\leq \ell_{t,i_t}$ and  
\begin{align*}
\mathbb{E}\left[ \sum_{t=1}^T \ell_{t,i_t} \right] - \sum_{t=1}^T \ell_{t,i^*} = \mathcal{O}\left( \sqrt{(K\ln T)\mathbb{E}\left[ \sum_{t=1}^T \ell_{t,i_t} \right]} + K\ln T\right). 
\end{align*}
Solving for $\sqrt{\mathbb{E}\left[ \sum_{t=1}^T\ell_{t,i_t} \right]}$ and rearranging then give
\begin{align*}
\mathbb{E}\left[ \sum_{t=1}^T \ell_{t,i_t} \right] - \sum_{t=1}^T \ell_{t,i^*} = \mathcal{O}\left( \sqrt{(K\ln T)\sum_{t=1}^T \ell_{t,i^*} } + K\ln T\right)=\mathcal{O}\left(\sqrt{KL_{T,i^*}\ln T}+K\ln T\right). 
\end{align*}
\end{proof}

\end{document}